\documentclass[sigconf]{acmart}
\usepackage{multirow}
\usepackage{multicol}
\usepackage{adjustbox}
\usepackage{graphicx}
\usepackage{caption}
\usepackage{subcaption}
\usepackage{algorithm}
\usepackage{algorithmic}
\AtBeginDocument{%
  \providecommand\BibTeX{{%
    \normalfont B\kern-0.5em{\scshape i\kern-0.25em b}\kern-0.8em\TeX}}}

\setcopyright{acmcopyright}
\copyrightyear{2018}
\acmYear{2018}
\acmDOI{10.1145/1122445.1122456}

\acmConference[Woodstock '18]{Woodstock '18: ACM Symposium on Neural
  Gaze Detection}{June 03--05, 2018}{Woodstock, NY}
\acmBooktitle{Woodstock '18: ACM Symposium on Neural Gaze Detection,
  June 03--05, 2018, Woodstock, NY}
\acmPrice{15.00}
\acmISBN{978-1-4503-XXXX-X/18/06}

\usepackage{amsmath,amsfonts,bm}

\def\eqref#1{equation~\ref{#1}}

\def\1{\bm{1}}

\def\ve{{\bm{e}}}

\def\vu{{\bm{u}}}

\def\vx{{\bm{x}}}

\def\mA{{\bm{A}}}

\def\mC{{\bm{C}}}

\def\mH{{\bm{H}}}
\def\mI{{\bm{I}}}

\def\mK{{\bm{K}}}
\def\mL{{\bm{L}}}

\def\mP{{\bm{P}}}
\def\mQ{{\bm{Q}}}

\def\mV{{\bm{V}}}
\def\mW{{\bm{W}}}
\def\mX{{\bm{X}}}
\def\mY{{\bm{Y}}}
\def\mZ{{\bm{Z}}}

\DeclareMathAlphabet{\mathsfit}{\encodingdefault}{\sfdefault}{m}{sl}
\SetMathAlphabet{\mathsfit}{bold}{\encodingdefault}{\sfdefault}{bx}{n}

\def\gA{{\mathcal{A}}}

\def\gC{{\mathcal{C}}}

\def\gG{{\mathcal{G}}}
\def\gH{{\mathcal{H}}}

\def\gN{{\mathcal{N}}}

\def\gP{{\mathcal{P}}}

\def\gX{{\mathcal{X}}}

\def\sG{{\mathbb{G}}}

\def\sS{{\mathbb{S}}}

\newcommand{\R}{\mathbb{R}}

\usepackage[utf8]{inputenc} 
\usepackage[T1]{fontenc}    
\usepackage{hyperref}       
\usepackage{cleveref}
\usepackage{url}            
\usepackage{booktabs}       
\usepackage{amsfonts}       
\usepackage{amsmath}
\usepackage{amsthm}
\usepackage{enumerate}
\usepackage{bbm}
\usepackage{graphicx}
\usepackage{caption}
\usepackage{subcaption}
\usepackage{nicefrac}       
\usepackage{microtype}      
\usepackage{xcolor}         
\usepackage{titlesec}
\usepackage{multirow}
\usepackage{multicol}
\usepackage{adjustbox}
\usepackage{pdfpages}

\newtheorem{definition}{Definition}[section]
\newtheorem{theorem}{Theorem}
\newtheorem{proposition}{Proposition}
\newtheorem{lemma}{Lemma}

\newcommand*{\ldblbrace}{\{\mskip-5mu\{}
\newcommand*{\rdblbrace}{\}\mskip-5mu\}}

\newtheorem{innercustomthm}{Theorem}
\newenvironment{customthm}[1]
  {\renewcommand\theinnercustomthm{#1}\innercustomthm}
  {\endinnercustomthm}

\begin{document}

\begin{abstract}

Graph Transformer has recently received wide attention in the research community with its outstanding performance, yet its structural expressive power has not been well analyzed. Inspired by the connections between Weisfeiler-Lehman (WL) graph isomorphism test and graph neural network (GNN), we introduce \textbf{SEG-WL test} (\textbf{S}tructural \textbf{E}ncoding enhanced \textbf{G}lobal \textbf{W}eisfeiler-\textbf{L}ehman test), a generalized graph isomorphism test algorithm as a powerful theoretical tool for exploring the structural discriminative power of graph Transformers. We theoretically prove that the SEG-WL test is an expressivity upper bound on a wide range of graph Transformers, and the representational power of SEG-WL test can be approximated by a simple Transformer network arbitrarily under certain conditions. With the SEG-WL test, we show how graph Transformers' expressive power is determined by the design of structural encodings, and present conditions that make the expressivity of graph Transformers beyond WL test and GNNs. Moreover, motivated by the popular shortest path distance encoding, we follow the theory-oriented principles and develop a provably stronger structural encoding method, Shortest Path Induced Subgraph (\textit{SPIS}) encoding. Our theoretical findings provide a novel and practical paradigm for investigating the expressive power of graph Transformers, and extensive synthetic and real-world experiments empirically verify the strengths of our proposed methods.

\end{abstract}

\title{On Structural Expressive Power of Graph Transformers}

\author{Wenhao Zhu}
\email{wenhaozhu@pku.edu.cn}
\orcid{1234-5678-9012}
\affiliation{%
  \institution{National Key Laboratory of General Artificial Intelligence, School of Intelligence Science and Technology, Peking University}
  \city{Beijing}
  \country{China}
}

\author{Tianyu Wen}
\email{tianyuwen@pku.edu.cn}
\affiliation{%
  \institution{Yuanpei College, Peking University}
  \city{Beijing}
  \country{China}
}

\author{Guojie Song}
\email{gjsong@pku.edu.cn}
\affiliation{%
  \institution{National Key Laboratory of General Artificial Intelligence, School of Intelligence Science and Technology, Peking University}
  \city{Beijing}
  \country{China}
}

\author{Liang Wang}
\email{liangbo.wl@alibaba-inc.com}
\affiliation{%
  \institution{Alibaba Group}
  \country{China}
}

\author{Bo Zheng}
\email{bozheng@alibaba-inc.com}
\affiliation{%
  \institution{Alibaba Group}
  \country{China}
}

\maketitle

\section{Introduction}
\label{sec1}

In the last decade, graph neural network (GNN) \citep{kipf2016semi,velivckovic2017graph} has become the prevalent neural architecture for deep learning on graph data. Following the message-passing scheme, GNNs learn the vector representation of node $v$ by iteratively aggregating and transforming features of its neighborhood nodes.
Recent studies \citep{xu2018powerful} have proved that Weisfeiler-Lehman (WL) graph isomorphism test can measure the theoretical expressive power of message-passing GNNs in distinguishing graph structures \citep{weisfeiler1968reduction}.

While in the last few years, the Transformer architecture \citep{vaswani2017attention} has achieved broad success in various machine learning tasks. On graph representation learning, though with higher complexity than GNNs, recent works \citep{ying2021transformers,kreuzer2021rethinking} have proved that graph Transformers can successfully model large-scale graph data and deliver state-of-the-art performance on real-world benchmarks. However, despite advances in empirical benchmark results, the theoretical expressive power of graph Transformers has not been deeply explored. Compared with GNN's message-passing strategy, which only includes neighborhood aggregation, most graph Transformers represent nodes by considering all pair-wise interactions in the input graph, meaning that every node has a \textbf{global receptive field} at each layer. Besides, since vanilla self-attention is ignorant of node ordering, like positional encodings in language models, graph Transformers must design various \textbf{structural encodings} as a soft inductive bias to leverage graph structural information. Therefore, previous methods like WL test can no longer be used to analyze the expressivity of graph Transformers, considering the substantial differences between two model architectures. The natural questions arise: \textit{How to characterize the structural expressive power of graph Transformers? How to build expressive graph Transformers that can outperform the WL test and GNNs?}

\begin{figure*}
\centering
\begin{subfigure}{.35\textwidth}
  \centering
  \includegraphics[width=\linewidth]{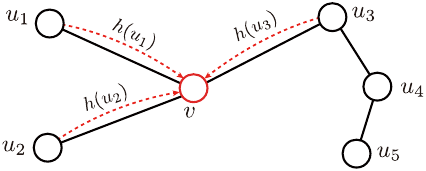}
  \caption{WL Test}
\end{subfigure}%
\begin{subfigure}{.35\textwidth}
  \centering
  \includegraphics[width=\linewidth]{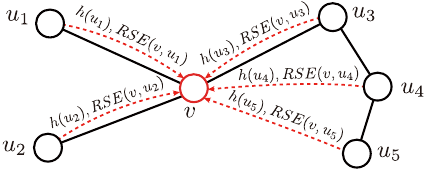}
  \caption{SEG-WL Test}
\end{subfigure}
\caption{An illustration of the node label update strategies of WL test and SEG-WL test.}
\label{fig1}

\end{figure*}

Our key to answering the questions above is \textbf{SEG-WL test} (\textbf{S}tructural \textbf{E}ncoding enhanced \textbf{G}lobal \textbf{W}eisfeiler-\textbf{L}ehman test), a generalized graph isomorphism test algorithm designed to characterize the expressivity of graph Transformer, as illustrated in Figure \ref{fig1}. Specifically, SEG-WL test represents a family of graph isomorphism test algorithms whose label update strategy is shaped by predefined \textit{structural encodings}. For every input graph, SEG-WL test first inserts \textit{absolute structural encodings} to the initial node labels. Then during each iteration, unlike WL test which updates the node label of $v$ by hashing  the multiset of its neighborhood node labels $\ldblbrace h(u):u\in \gN(v)\rdblbrace$, SEG-WL test \textit{globally} hashes $\ldblbrace(h(u),\text{RSE}(u,v)):u\in V\rdblbrace$, the collection of all node labels together with \textit{relative structural encodings} to the central node. We theoretically prove that SEG-WL test is an expressivity upper bound on any graph neural model that learns structural information via structural encodings, including most graph Transformers (Theorem \ref{thm1}). Moreover, with the universal approximation theorem of Transformers \citep{yun2019transformers}, we show under certain assumptions, the expressivity of SEG-WL test can be approximated at any precision by a simple Transformer network which incorporates relative structural encodings as attention biases (Theorem \ref{thm2}). These conclusions guarantee that SEG-WL test can be a solid theoretical tool for our deeper investigation into the expressivity of graph Transformers.

Since the label update strategy of SEG-WL test is driven by structural encoding, we next develop general theories to understand the characteristics of structural encodings better. Our central result shows that one can compare the expressivity and convergence rate of SEG-WL tests by looking into the relationship between their structural encodings (Theorem \ref{thm3}), which provides us with a simple and powerful solution to analyze the representational capacity of SEG-WL test and graph Transformers. We show WL test can be viewed as a nested case of SEG-WL test (Theorem \ref{thm4}), and theoretically characterize how to design structural encodings that make graph Transformers more expressive than WL test and GNNs. We demonstrate that graph Transformers with the shortest path distance (\textit{SPD}) structural encodings (like Graphormer \citep{ying2021transformers}) are strictly more powerful than the WL test (Theorem \ref{thm5}), and they have distinctive expressive power that differs from encodings that focus on local information (Proposition \ref{mpro1}). Based on \textit{SPD} encodings, we follow the theoretical guidelines and design \textit{SPIS}, a provably more powerful structural encoding (Theorem \ref{thm6}) with profound representational capabilities (Proposition \ref{mpro2}-\ref{mpro3}). Our synthetic experiments verify that \textit{SPIS} has remarkable expressive power in distinguishing graph structures, and the performances of existing graph Transformers can be consistently improved when equipped with the proposed \textit{SPIS}.


\paragraph{\textbf{Contributions.}} We summarize the main contributions of this work as follows:
\begin{itemize}
    \item We introduce the SEG-WL test algorithm and prove it well characterizes the expressive power of various graph Transformers (Section \ref{sec3}, Theorem \ref{thm1}-\ref{thm2}).
    \item Using the SEG-WL test, we develop a generalized theoretical framework on \textit{structural encodings} that determines the expressivity of graph Transformers, and show how to make graph Transformers more expressive than WL test and GNNs (Section \ref{sec4}, Theorem \ref{thm3}-\ref{thm4}).
    \item We conduct in-depth investigation into the expressivity of the existing \textit{SPD} structural encoding, and propose a provably more powerful encoding method \textit{SPIS} (Section \ref{sec5}, Theorem \ref{thm5}-\ref{thm6}).
    \item Synthetic and real-world experiments demonstrate that \textit{SPIS} has strong expressive power in distinguishing graph structures, and performances of benchmark graph Transformers are dominated by the theoretically more powerful \textit{SPIS} encoding (Section \ref{sec6}).
\end{itemize}

Overall, we build a general theoretical framework for analyzing the expressive power of graph Transformers, and propose the \textit{SPIS} structural encoding to push the boundaries of both expressivity and performance of graph Transformers.


\section{Related Work}
\label{asec2}
\subsection{WL Test and GNNs}
\paragraph{\textbf{Weisfeiler-Lehman Graph Isomorphism Test.}} The Weisfeiler-Lehman test is a hierarchy of graph isomorphism tests \citep{weisfeiler1968reduction,grohe2017descriptive}, and the 1-WL test is know to be an upper bound on the expressivity of message-passing GNNs \citep{xu2018powerful}. Note that in this paper, without further notations, we will use the term \textit{WL} to refer to 1-WL test. Formally, the definition of WL test is presented as

\begin{definition}[WL test]
Let the input be a labeled graph $G=(V,E)$ with label map $h_0:V\to\gX$. WL test iteratively updates node labels of $G$, where at the $t$-th iteration, the updated node label map $w_t: V\to\gX$ is computed as
\begin{align}
    w_t(v)=\Phi\left(w_{t-1}(v),\ldblbrace w_{t-1}(u):u\in\gN(v)\rdblbrace\right), \label{wl}
\end{align}
where $w_0=h_0$ and $\Phi$ is a function that injectively maps the collection of all possible tuples in the r.h.s. of Equation \ref{wl} to $\gX$. We say two graphs $G_1,G_2$ are distinguished as non-isomorphic by WL test if after $t$ iterations, the WL test generates $\ldblbrace w_{t}(v)|v\in V_1\rdblbrace\neq\ldblbrace w_{t}(v)|v\in V_2\rdblbrace$ for some $t$.
\end{definition}

\paragraph{\textbf{GNNs beyond the Expressivity of 1-WL.}} Since standard GNNs (like GCN \citep{kipf2016semi}, GAT \citep{velivckovic2017graph} and GIN \citep{xu2018powerful}) have expressive power bounded by the 1-WL, many works have proposed to improve the expressivity of GNNs beyond the 1-WL. High-order GNNs including \citep{morris2019weisfeiler,maron2019provably,azizian2020expressive,morris2020weisfeiler} build graph neural networks inspired from k-WL with $k>3$ to acquire the stronger expressive power, yet they mostly have high computational costs and complex network designs. Some works have proposed to use pre-computed topological node features to enhance the expressive power of GNNs, including \citep{monti2018motifnet,liu2020neural,bouritsas2022improving}. These additional features may contain the number of the appearance of certain substructures like triangles, rings and circles. And recent works like \citep{you2021identity,vignac2020building,sato2021random,wijesinghe2021new} show that the expressivity of GNNs can also be enhanced using random node identifiers or improved message-passing schemes.

\subsection{Graph Transformer}

\paragraph{\textbf{The Transformer Architecture.}} Transformer is first proposed in \cite{vaswani2017attention} to model sequence-to-sequence functions on text data, and now has become the prevalent neural architecture for natural language processing \citep{devlin2018bert}. A Transformer layer mainly consists of a multi-head self-attention (MHA) module and a position-wise feed-forward network (FFN) with residual connections. For queries $\mQ\in\R^{n_q\times d}$, keys $\mK\in\R^{n_k\times d}$ and values $\mV\in\R^{n_k\times d}$, the scaled dot-product attention module can be defined as
\begin{align}
    \text{Attention}(\mQ,\mK,\mV)=\text{softmax}(\mA)\mV,\mA=\frac{\mQ\mK^{\top}}{\sqrt{d}},
\end{align}
where $n_q, n_k$ are number of elements in queries and keys, and $d$ is the hidden dimension. Then, the multi-head attention is calculated as
\begin{align}
    &\text{MHA}(\mQ,\mK,\mV)=\text{Concat}(\text{head}_1,\ldots,\text{head}_h)\mW^O,\\
    &\text{head}_i=\text{Attention}(\mQ\mW_i^Q,\mK\mW_i^K,\mV\mW_i^V),\text{for }i=1,\ldots,h,
\end{align}
where $h$ is number of attention heads, $\mW_i^Q\in\R^{d\times d_k},\mW_i^K\in\R^{d\times d_k},$ $\mW_i^V\in\R^{d\times d_v}$ and $\mW^O\in\R^{hd_v\times d}$ are projection parameter matrices, $d,d_k,d_v$ are the dimension of hidden layers, keys and values. In encoder side of the original Transformer architecture, all queries, keys and values come from the input sequence embeddings.

After multi-head attention, the position-wise feed-forward network is applied to every element in the sequence individually and identically. This network is composed of two linear transformations, an activation function and residual connections in between. Layer normalization \cite{ba2016layer} is also performed before the multi-head self-attention and feed-forward network \cite{xiong2020layer}. A Transformer layer can be defined as below: 
\begin{align}
    &\text{Transformer}(\mQ,\mK,\mV)=\text{FFN}(\text{LN}(\mH))+\mH,\\
    &\mH=\text{MHA}(\text{LN}(\mQ,\mK,\mV))+\mQ.
\end{align}

\paragraph{\textbf{Graph Transformers.}} Along with the recent surge of Transformer, many prior works have attempted to bring Transformer architecture to the graph domain, including GT \cite{dwivedi2020generalization}, GROVER \cite{rong2020self}, Graphormer \cite{ying2021transformers}, SAN \cite{kreuzer2021rethinking}, SAT \cite{chen2022structure}, ANS-GT \cite{zhang2022hierarchical}, GraphGPS \cite{rampasek2022GPS}, GRPE \cite{park2022grpe}, EGT \cite{hussain2022global} and NodeFormer \cite{wunodeformer}. These methods generally treat input graph as a sequence of node features, and apply various methods to inject structural information into the network. GT \cite{dwivedi2020generalization} provides a generalization of Transformer architecture for graphs with modifications like using Laplacian eigenvectors as positional encodings and adding edge feature representation to the model. GROVER \cite{rong2020self} is a molecular large-scale pretrain model that applies Transformer to node embeddings calculated by GNN layers. Graphormer \cite{ying2021transformers} proposes an enhanced Transformer with centrality, spatial and edge encodings, and achieves state-of-the-art performance on many molecular graph representation learning benchmarks. SAN \cite{kreuzer2021rethinking} presents a learned positional encoding that cooperates with full Laplacian spectrum to learn the position of each node in the graph. Gophormer \cite{zhao2021gophormer} applies structural-enhanced Transformer to sampled ego-graphs to improve node classification performance and scalability. GraphGPS \citep{rampasek2022GPS} proposes a recipe on how to build a general, powerful, scalable (GPS) graph Transformer with linear complexity and state-of-the-art results on real benchmark tests. SAT \citep{chen2022structure} proposes the Structure-Aware Transformer with its new self-attention mechanism which incorporates structural information into the original self-attention by extracting a subgraph representation rooted at each node using GNNs before computing the attention.

\section{Preliminaries}
\paragraph{\textbf{Basic Notations.}} Let $G=(V,E)$ be a undirected graph where $V=\{v_1,v_2,\ldots,v_n\}$ is the node set that consists of $n$ nodes, and $E\subset V\times V$ is edge set. Let $h_0:V\to\gX$ defines the input feature vector (or label) attached to nodes, where $\gX\subset \R^d$ is the feature space. In this paper, we only consider simple undirected graphs with node features, and we use $\gG$ to denote the set of all possible labeled simple undirected graphs.

\paragraph{\textbf{Structural Encodings.}} Generally, \textit{structural encoding} is a function that encodes structural information in  $G$ to numerical vectors associated with nodes or node tuples of $V$. In the scope of this paper, we mainly use two types of structural encodings: \textit{absolute structural encoding} (ASE), which represents absolute structural knowledge of individual nodes, and \textit{relative structural encoding} (RSE), which represents the relative structural relationship between two nodes in the entire graph context. For a certain graph Transformer model, its structural encoding scheme consists of both absolute and relative encodings, and we present the formal definition below:

\begin{definition}[Structural Encoding]
A \textbf{structural encoding scheme} $S=(f_A,f_R)$ is a pair of functions, where for any graph $G=(V,E)$, $f_A(v,G)\in\gC$ is the absolute structural encoding of any node $v\in V$, $f_R(v,u,G)\in\gC$ is the relative structural encoding of any node pair $(v,u)\in V\times V$, and $\gC$ is the target space. A structural encoding scheme is called \textbf{regular} if the relative structural encoding function satisfies $f_R(v,v,G)\neq f_R(v,u,G)$ for $u,v\in V$ and $u\neq v$. 
\end{definition}

For example, we can use degree as an absolute structural encoding of a node, and use the shortest path distance between two nodes as the relative structural encoding of a node pair. We will discuss structural encodings more in the following sections. 

\section{SEG-WL Test and Graph Transformers}
\label{sec3}

In this section, we mathematically formalize the SEG-WL test algorithm and theoretically prove that SEG-WL test well characterizes the expressive power of graph Transformers. Note that Appendix \ref{asec1} provides detailed proofs for all theorems and propositions in the following sections.

\subsection{From WL Test to SEG-WL Test} 

Generally, previous GNN-based methods represent a node by summarizing and transforming its neighborhood information. This strategy leverages graph structure in a \textit{hard-coded} way, where the structural knowledge is reflected by removing the information exchange between non-adjacent nodes. WL test is a high-level abstraction of this learning paradigm. However, graph Transformers take a fundamentally different way of learning graph representations. Without any hard inductive bias, self-attention represents a node by aggregating its semantic relation between every node in the graph, and structural encodings guide this aggregation as a \textit{soft} inductive bias to reflect the graph structure. The proposed SEG-WL test then becomes a generalized algorithm for this powerful and flexible learning scheme by updating node labels based on the entire label set of nodes and their relative structural encoding to the central node, defined as follows:  

\begin{definition}[SEG-WL Test]
Let the input be a labeled graph $G=(V,E)$ with label map $h_0:V\to\gX$. For structural encoding scheme $S=(f_A,f_R)$, its corresponding SEG-WL test algorithm first computes the initial label mapping $g_0:V\to\gX$ by adding the absolute structural encodings:
\begin{align}
    g_0(v)=\Phi_0(h_0(v),f_A(v,G)),
\end{align}
where $\Phi_0$ is a injective function that maps the tulple to $\gX$. Then SEG-WL test iteratively updates node labels of $G$, where at the $t$-th iteration, the updated node label mapping $g_t: V\to\gX$ is computed as
\begin{align}
    g_t(v)=\Phi\left(\ldblbrace (g_{t-1}(u),f_R(v,u,G)):u\in V\rdblbrace\right), \label{gt}
\end{align}
where $\Phi$ is a function that injectively maps the collection of all possible multisets of tuples in the r.h.s. of Equation \ref{gt} to $\gX$. We say two graphs $G_1,G_2$ are distinguished as non-isomorphic by $S$-SEG-WL test if after $t$ iterations, $S$-SEG-WL generates $\ldblbrace g_{t}(v):v\in V_1\rdblbrace\neq\ldblbrace g_{t}(v):v\in V_2\rdblbrace$ for some $t$.
\end{definition}
Note that for structural encoding scheme $S$ we use $S$-SEG-WL to denote its corresponding SEG-WL test algorithm. Following its definition, we will show that SEG-WL test characterizes a wide range of graph neural models that leverage graph structure as a soft inductive bias:

\begin{theorem}
For any structural encoding scheme $S=(f_A,f_R)$ and labeled graph $G=(V,E)$ with label map $h_0:V\to\gX$, if a graph neural model $\gA:\gG\to\R^d$ satisfies the following conditions:
\begin{enumerate}
\itemsep0em
    \item $\gA$ computes the initial node embeddings with
    \begin{align}
        l_0(v)=\phi(h_0(v),f_A(v,G)),
    \end{align}
    \item $\gA$ aggregates and updates node embeddings iteratively with
    \begin{align}
        l_t(v)=\sigma(\ldblbrace(l_{t-1}(u),f_R(v,u,G)):u\in V\rdblbrace),
    \end{align}
    where $\phi$ and $\sigma$ above are model-specific functions,
    \item The final graph embedding is computed by a global readout on the multiset of node features $\ldblbrace l_t(v):v\in V\rdblbrace$.
\end{enumerate}
then for any labeled graphs $G_1$ and $G_2$, if $\gA$ maps them to different embeddings, $S$-SEG-WL also decides $G_1$ and $G_2$ are not isomorphic.
\label{thm1}
\end{theorem}

In the SEG-WL test framework outlined by Theorem \ref{thm1}, Appendix \ref{asec4} presents examples of characterizing the expressivity of existing graph Transformer models using certain structural encoding , including \cite{dwivedi2020generalization,ying2021transformers,kreuzer2021rethinking,zhao2021gophormer,chen2022structure}  . Notably, in Appendix \ref{asec11} we provide a more generalized version of Theorem \ref{thm1} which proves that the widely adopted virtual node trick \citep{ying2021transformers} has no influence on the maximum model expressive power.

\subsection{Theoretically Powerful Graph Transformers}
Though the maximum representational power of most graph Transformer models has been well characterized by SEG-WL test, it is still unknown if there exists a graph Transformer model that can reach its expressivity upper bound. Transformer layers are composed of self-attention module and feed-forward network,  which drive them much more complex than standard GNN layers, making it challenging to analyze the expressive properties of graph Transformers. Thanks to the universal approximation theorem of Transformers \citep{yun2019transformers}, our next theoretical result demonstrates that under certain conditions, a simple graph Transformer model which leverages relative structural encodings as attention biases via learnable embedding layers (named as bias-GT) can arbitrarily approximate the SEG-WL test iterations for any structural encoding design:

\begin{theorem}
For any regular structural encoding scheme $S$, graph order $n$, $1<p<\infty$ and $\epsilon>0$, let $f_t$ represent the function of $S$-SEG-WL with $t$ iterations. Then $f_t$ can be approximated by a bias-GT network $g$ with $S$ such that $\mathsf{d}_p(f_t,g)<\epsilon$ if (i) the feature space $\gX$ is compact, (ii) $\Phi$ can be extended to a continuous function with respect to node labels.
\label{thm2}
\end{theorem}

In Theorem \ref{thm2} we define $f_t$ by stacking all labels generated by SEG-WL test with $t$ iterations, and $\mathsf{d}_p(f_t,g)$ is the maximum $\ell^p$ distance between $f_t$ and $g$ when changing the input graph structure. Proof for Theorem \ref{thm2} and the detailed descriptions for $f_t,g,\mathsf{d}_p,\Phi$ and the bias-GT network are provided in Appendix \ref{asec12}.

Under certain conditions, Theorem \ref{thm2} guarantees that the simple Transformer network bias-GT is theoretically capable of capturing structural knowledge introduced as attention biases and arbitrarily approximating its expressivity upper bound, though a good approximation may require many Transformer layers. Overall, considering that the simple bias-GT network (which can be viewed as a simplification of existing graph Transformers like Graphormer \citep{ying2021transformers}) is one instance among the most theoretically powerful graph Transformers, one can translate the central problem of characterizing the expressive capacity of graph Transformers into understanding the expressivity of SEG-WL test, which is determined by the design of structural encodings.

\section{General Discussions on SEG-WL Test and Structural Encodings}
\label{sec4}
In this section, we develop a unified theoretical framework for analyzing structural encodings and the expressivity of SEG-WL test. One can tell that each SEG-WL test iteration has quadratic complexity with respect to the graph size and is more computationally expensive than WL, yet we will prove in the following text that SEG-WL test could exhibit extraordinary expressive power and lower necessary iterations when combined with a variety of structural encodings. We first present concrete examples and show how the expressivity of structural encodings can be compared. Based on these findings, we prove that WL test is a nested case of SEG-WL test and theoretically characterize how to design structural encodings exceeding the expressivity of WL test. More discussions are provided in Appendix \ref{asec3}.

\subsection{Examples of Structural Encodings}
\label{sec41}
\paragraph{\textbf{Identical Encoding.}} The simplest encoding scheme assigns identical information to every node and non-duplicated node pair. Formally, let $\textit{id}=(\textit{id}_A,\textit{id}_R)$ be the identical encoding scheme, then for $G=(V,E)$ and $v,u\in V$, $\textit{id}_A(v,G)=0,\textit{id}_R(v,u)=1,\textit{id}_R(v,v)=0$.

\paragraph{\textbf{Node Degree Absolute Encoding.}} A common strategy for injecting absolute structural knowledge to node embeddings in the entire graph context is using the node degree as an additional signal. For graph $G=(V,E)$ and $v\in V$, let $\textit{Deg}_A(v,G)$ be the degree of node $v$, then $\textit{Deg}_A$ is the node degree absolute encoding function. 

\paragraph{\textbf{Neighborhood Relative Encoding.}} Neighborhood relative encoding $\textit{Neighbor}_R$ is a basic example that encodes edge connections. For $G=(V,E)$ and $v,u\in V$, it is defined as
\begin{align}
    \textit{Neighbor}_R(v,u,G)=
    \begin{cases}
    1,\text{ if }(v,u)\in E,\\
    2,\text{ if }(v,u)\notin E,
    \end{cases}
\end{align}
and $\textit{Neighbor}_R(v,v,G)=0.$ We also use $\textit{Neighbor}=(\textit{id}_A,\textit{Neighbor}_R)$ to denote the encoding scheme that combines $\textit{Neighbor}_R$ with identical absolute encoding. Intuitively, we will show that $\textit{Neighbor}$ precisely shapes the expressivity of WL test.

\paragraph{\textbf{Shortest Path Distance Relative Encoding.}} First introduced by \citep{ying2021transformers}, shortest path distance (SPD) is a popular choice for representing relative structural information between two nodes in the graph. We formulate it as

\begin{align}
    \textit{SPD}_R(v,u,G)=
    \begin{cases}
    &\text{the SPD between $v$ and $u$ in $G$,}\\
    &\text{\quad\quad if $v$ and $u$ are connected,}\\
    &\infty,\text{\quad if $v$ and $u$ are not connected,}
    \end{cases}
\end{align}
where $\infty$ can be viewed as an element in $\gC$ and $\textit{SPD}_R(v,v,G)=0$. We also define the \textit{SPD} structural encoding scheme as $\textit{SPD}=(\textit{id}_A,\textit{SPD}_R)$.

\subsection{Structural Encoding Determines the Expressiveness and Convergence Rate of SEG-WL test}

Our next theoretical result is based on the intuitive idea that if one can infer the structural information in scheme $S$ from another encoding scheme $S'$, then $S'$ should be generally more powerful and converge faster on graphs as it contains more information. To formulate this theoretical insight, we start by defining a partial ordering to characterize the relative discriminative power of structural encodings:

\begin{definition}[Partial Order Relation on Structural Encodings]
For two structural encoding schemes $S=(f_A,f_R)$ and $S'=(f_A',f_R')$, we call $S'\succeq S$ if there exist mappings $p_A,p_R$ such that for any $G=(V,E)$ and $v,u\in V$ we have
\begin{align}
    &f_A(v,G)=p_A(f_A'(v,G)), \label{thm3eq1}\\
    &f_R(v,u,G)=p_R(f_R'(v,u,G)).\label{thm3eq2}
\end{align}
\end{definition}
With the definition above, we next present the central theorem that shows structural encoding determines the expressiveness and convergence rate of SEG-WL test:

\begin{theorem}
For two structural encoding schemes $S$ and $S'$, if $S'\succeq S$, then

\begin{enumerate}[\hspace{0.7cm}(1)]
    \item $S'$-SEG-WL is more expressive than $S$-SEG-WL in testing non-isomorphic graphs.\footnote{For two isomorphic testing algorithms $A$ and $B$, we say $A$ is more expressive than $B$ if any non-isomorphic graphs distinguishable by $B$ can be distinguished by $A$.}
    \item for a pair of graphs $G_1$ and $G_2$ that $S$-SEG-WL distinguishes as non-isomorphic after $t$ iterations, $S'$-SEG-WL can distinguish $G_1$ and $G_2$ as non-isomorphic within $t$ iterations.
\end{enumerate}
\label{thm3}
\end{theorem}

Theorem \ref{thm3} lays out a critical fact on the relations between SEG-WL test and structural encodings: \textit{if $S'\succeq S$, then compared with $S$-SEG-WL, $S'$-SEG-WL is more powerful in graph isomorphism testing and will always converge faster when testing graphs.} Through Theorem \ref{thm3} , we can distinguish  the expressive power of various structural encodings by comparing them with baseline encodings defined in Section \ref{sec41}. Given existing structural encodings, Theorem \ref{thm3} shows that more powerful encodings can be developed by adding extra non-trivial structural information. We will elaborate on the ideas above in the following text.

\subsection{WL as SEG-WL Test}
The first application of our theoretical results is to answer the question: \textit{How to design graph Transformers that are more powerful than the WL test?} Since the expressivity of graph Transformers depends on the corresponding SEG-WL test, we first characterize WL test as a special case of SEG-WL test:
\begin{theorem}

Two non-isomorphic graphs can be distinguished by WL if and only if they are distinguishable by \textit{Neighbor}-SEG-WL.
\label{thm4}
\end{theorem}
Theorem \ref{thm4} proves that though $\textit{Neighbor}$-SEG-WL hashes the whole set of node labels, its expressivity is still exactly the same as WL test. Therefore, from a theoretical perspective, graph Transformer models with \textit{Neighbor} encoding have the same expressive power as WL-GNNs, though they feature the multi-head attention mechanism and global receptive field for every node. Combined with Theorem \ref{thm3}, the answer to the question above becomes simple: \textit{To design a graph Transformer that is more powerful than the WL test, we only need to equip it with structural encoding more expressive than \textit{Neighbor}.}

Furthermore, considering many GNNs utilize absolute structural encodings to enhance their expressive power (e.g.,  \cite{bouritsas2022improving}), we wonder how to compare their expressiveness against Transformers. For any absolute structural encoding $f_A$, we can easily infer from Theorem \ref{thm4} that $f_A$-WL (WL with additional node features generated by $f_A$) is equivalent to $(f_A,\textit{Neighbor}_R)$-SEG-WL on expressive power. Therefore, to develop graph Transformers with expressivity beyond WL-GNNs, it is necessary to design relative structural encodings that are more powerful than $\textit{Neighbor}_R$.

\section{Shortest-Path-Based Relative Structural Encodings}
\label{sec5}
This section presents an example of utilizing  our theory and designing powerful relative structural encodings for graph Transformers. We start from encodings based on the shortest path between two nodes, like \textit{SPD} used in Graphormer \citep{ying2021transformers}.

\subsection{Expressivity of \textit{SPD} Encoding}
\label{sec51}
Considering that two nodes are adjacent when SPD between them is 1, we can easily conclude that $\textit{SPD}_R\succeq \textit{Neighbor}_R$. Therefore, it can be inferred from Theorem \ref{thm3} that \textit{SPD}-SEG-WL is more powerful than WL. Besides, we can find many pairs of non-isomorphic graphs indistinguishable by WL but not for \textit{SPD}-SEG-WL. We have
\begin{theorem}
(1) SPD-SEG-WL is \textbf{strictly} more expressive than WL in testing non-isomorphic graphs\footnote{For two isomorphic testing algorithms $A$ and $B$, we say $A$ is strictly more expressive than $B$ if $A$ is more expressive than $B$ in testing non-isomorphic graphs, and there exist non-isomorphic graphs $G_1$ and $G_2$ such that $A$ can distinguish $G_1$ and $G_2$ but not for $B$.};

(2) For $G_1$ and $G_2$ that WL distinguishes as non-isomorphic after $t$ iterations, \textit{SPD}-SEG-WL can distinguish $G_1$ and $G_2$ as non-isomorphic within $t$ iterations.
\label{thm5}
\end{theorem}

\begin{proof}
We can easily show that \textit{SPD}-SEG-WL is more powerful than \textit{Neighbor}-SEG-WL using Theorem \ref{thm3} since two nodes are linked if there shortest path distance is 1. And according to Theorem \ref{thm4}, \textit{Neighbor}-SEG-WL is as powerful as WL, then \textit{SPD}-SEG-WL is more powerful than WL.

Figure \ref{afig1} below shows a pair of graphs that can be distinguished by \textit{SPD}-SEG-WL but not WL, which completes the proof.
\end{proof}

\begin{figure}[h]
\centering
\includegraphics[width=0.5\linewidth]{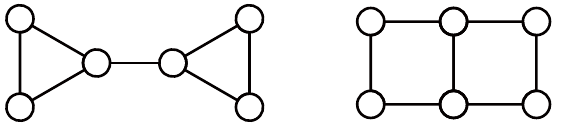}
\caption{Two graphs that can be distinguished by \textit{SPD}-SEG-WL but not WL.}
\label{afig1}
\end{figure}

Theorem \ref{thm5} formally proves that \textit{SPD}-SEG-WL is strictly more powerful and converges faster than WL in graph isomorphism testing. In addition to Theorem \ref{thm5}, we want to find out how the global structural information leveraged by shortest path encodings affects the discriminative power of SEG-WL test. 
We introduce the concept of \textit{receptive field} of structural encodings, that when $S$ has $k$-hop receptive field, any structural information encoded by $S$ only depends on the $k$-hop neighborhood of the central node.
For example, \textit{Neighbor} has $1$-hop receptive field because only neighborhood connections are considered by \textit{Neighbor} encoding. However, the receptive field of \textit{SPD} is not restricted to $k$-hop for any $k$, since we can construct graphs with SPD between two nodes arbitrarily large. We show this global-aware receptive field brings distinctive power to \textit{SPD} that differs from any encodings with local receptive field, in following Proposition \ref{mpro1}:
\begin{proposition}
For any $k$ and any structural encoding scheme $S$ with $k$-hop receptive field, there exists a pair of graphs that \textit{SPD}-SEG-WL can distinguish, but $S$-SEG-WL can not. 
\label{mpro1}
\end{proposition}

\begin{proof}
Let $C_l$ denote the cycle graph of length $l$. Then consider two graphs $G_1$ and $G_2$, where $G_1$ consists of $2k+4$ identical $C_{2k+3}$ graphs, and $G_2$ consists of $2k+3$ identical $C_{2k+4}$ graphs. $G_1$ and $G_2$ have the same number of nodes, and the induced $k$-hop neighborhood of any node in either of the two graphs is simply a path of length $2k+1$. As a result, for structural encoding scheme $S$ with $k$-hop receptive field, $S$-SEG-WL generates identical labels for every node in the two graphs, making $G_1$ and $G_2$ indistinguishable for $S$-SEG-WL. However, in $G_2$ there exists shortest paths of length $k+2$ while $G_1$ not, so \textit{SPD}-SEG-WL can distinguish the two graphs.
\end{proof}

Though \textit{SPD} has its unique expressive power and is more powerful than WL, many low-order non-isomorphic graphs remain to be indistinguishable by \textit{SPD}-SEG-WL (see Proof for Theorem \ref{thm6}), which leads us to find encodings that are more powerful than \textit{SPD}. Following Theorem \ref{thm3}, building structural encoding $S$ that satisfies $S\succeq \textit{SPD}$ can be done by adding meaningful information to \textit{SPD}, which illustrates the motivation for \textit{SPIS} we will next introduce. 

\subsection{\textit{SPIS} Relative Structural Encoding}
From the perspective of graph theory, for two connected nodes $v,u$ in the graph, there can be multiple shortest paths connecting $v$ and $u$, and these shortest paths may be linked or have overlapping nodes. Since \textit{SPD} only encodes the length of shortest paths, one intuitive idea is to enhance it with features characterizing the rich structural interactions between different shortest paths. Inspired by concepts like betweenness centrality in network analysis \citep{freeman1977set}, we propose the concept of shortest path induced subgraph (SPIS) to characterize the structural relations between nodes on shortest paths:
\begin{definition}[Shortest Path Induced Subgraph]
For $G=(V,E)$ and $v,u\in V$, $\text{SPIS}(v,u)=(V_{\text{SPIS}(v,u)},E_{\text{SPIS}(v,u)})$, the \textbf{shortest path induced subgraph} between $v$ and $u$ is an induced subgraph of $G$, where
\begin{align}
    V_{\text{SPIS}(v,u)}=\{s:s\in V \text{ and }\text{SPD}_R(v,s)+\text{SPD}_R(s,u)=\text{SPD}_R(v,u)\}.
\end{align}
\end{definition}
$\text{SPIS}(v,u)$ is an induced subgraph of $G$ that contains all nodes on shortest paths between $v$ and $u$. To encode knowledge in SPIS as numerical vectors, we propose the relative encoding method $\textit{SPIS}_R$ by enhancing $\textit{SPD}_R$ with the total numbers of nodes and edges of SPIS between nodes, as
\begin{align}
    \textit{SPIS}_R(v,u,G)=(\textit{SPD}_R(v,u,G),|V_{\textit{SPIS}(v,u)}|,|E_{\textit{SPIS}(v,u)}|),
\end{align}
and we define the structural encoding scheme $\textit{SPIS}=(\textit{id}_A,\textit{SPIS}_R)$.

\subsection{Analysis on \textit{SPIS} Encoding}
In the following, we will analyze the proposed \textit{SPIS} encoding and characterize its mathematical properties, comparing it with \textit{SPD} and WL. To start with, as \textit{SPIS} is constructed by adding information to \textit{SPD}, we have $\textit{SPIS}\succeq \textit{SPD}$ and it is be more powerful than \textit{SPD}-SEG-WL according to Theorem \ref{thm3}.
\begin{theorem}
(1) SPIS-SEG-WL is \textbf{strictly} more expressive than SPD-SEG-WL in testing non-isomorphic graphs.

(2) For $G_1$ and $G_2$ that SPD-SEG-WL distinguishes as non-isomorphic after $t$ iterations, \textit{SPIS}-SEG-WL can distinguish $G_1$ and $G_2$ as non-isomorphic within $t$ iterations.
\label{thm6}
\end{theorem}

\begin{proof}
Considering $\textit{SPD}_R$ is the first dimension of $\textit{SPIS}_R$, we have $\textit{SPIS}\succeq\textit{SPD}$ and we can prove $\textit{SPIS}$-SEG-WL is more powerful than $\textit{SPD}$-SEG-WL according to Theorem \ref{thm3}.

Figure \ref{afig2} below shows a pair of graphs that can be distinguished by \textit{SPIS}-SEG-WL but not \textit{SPD}-SEG-WL. It is trivial to verify that \textit{SPD}-SEG-WL can not distinguish them. For \textit{SPIS}-SEG-WL, to understand this, Figure \ref{afig2} colors examples of SPIS between non-adjacent nodes in the two graphs, where the nodes at two endpoints are colored as red. In the first graph, every SPIS between non-adjacent nodes has 3 nodes, but in the second graph there exists SPIS between non-adjacent nodes that has 4 nodes, so \textit{SPIS}-SEG-WL can distinguish them.
\end{proof}

\begin{figure}[h]
\centering
\includegraphics[width=0.5\linewidth]{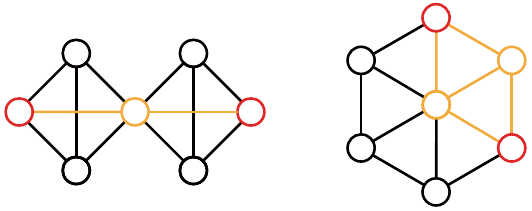}
\caption{Two graphs that can be distinguished by \textit{SPIS}-SEG-WL but not \textit{SPD}-SEG-WL.}
\label{afig2}
\end{figure}

Next, we show that \textit{SPIS}-SEG-WL exhibits far superior performance to WL and \textit{SPD}-SEG-WL on important graph structures. The computational complexity of \textit{SPIS} is discussed in Appendix \ref{asec3}.

\paragraph{\textbf{\textit{SPIS}-SEG-WL Distinguishes All Low-order Graphs ($n\leq 8$).}}
On low-order graphs, our synthetic experiments in Table \ref{tbl1} confirm that \textit{SPIS}-SEG-WL distinguishes \textit{all} non-isomorphic graphs with order equal to or less than 8, which is much more powerful than WL with 332 indistinguishable pairs and \textit{SPD}-SEG-WL with 200 indistinguishable pairs. This strong discriminative power on low-order graphs shows that \textit{SPIS} can accurately distinguish local structures in real-world graphs.

\paragraph{\textbf{\textit{SPIS}-SEG-WL Well Distinguishes Strongly Regular Graphs.}} A regular graph is a graph parameterized by two parameters $n,k$ which has $n$ nodes and each node has the $k$ neighbors, denoted as $\text{RG}(n,k)$. And a strongly regular graph parameterized by four parameters $(n,k,\lambda,\mu)$ is a regular graph $\text{RG}(n,k)$ where every adjacent pair of nodes has the same number $\lambda$ of neighbors in common, and every non-adjacent pair of nodes has the same number $\mu$ of neighbors in common, denoted as $\text{SRG}(n,k,\lambda,\mu)$.

Due to their highly symmetric structure, regular graphs are known to be failure cases for graph isomorphism test algorithms. For example, WL can not discriminate any regular graphs of the same parameters, making any pair of strongly regular graphs with the same $n$ and $k$ indistinguishable to it, even $\lambda$ and $\mu$ could be different. Yet Proposition \ref{mpro2} guarantees that $\textit{SPIS}$-SEG-WL can distinguish any pair of strongly regular graphs of different parameters:
\begin{proposition}
\textit{SPIS}-SEG-WL can distinguish any pair of strongly regular graphs of different parameters.
\label{mpro2}
\end{proposition}

\begin{proof}
It is trivial to verify that regular graphs with different parameters can be distinguished by WL, so we focus on strongly regular graphs with the same $n$ and $k$ but different $\lambda$ and $\mu$. For $\text{SRG}(n,k,\lambda,\mu)$, since every non-adjacent pair of nodes has $\mu$ neighbors in common, the SPIS between evry non-adjacent pair of nodes will have $\mu+2$ nodes, which implies that \textit{SPIS}-SEG-WL can distinguish strongly regular graphs with different $n,k,\mu$. Besides, the four parameters of strongly regular graphs are not independent, they satisfy
\begin{align}
    \lambda=k-1-\frac \mu k(n-k-1),
\end{align}
so \textit{SPIS}-SEG-WL can distinguish strongly regular graphs with different parameters.
\end{proof}

It is worth mentioning that, for strongly regular graphs with the same parameters, \textit{SPIS} also exhibits outstanding discriminative power, with the number of total failures being far less than WL and \textit{SPD}-SEG-WL (See Section \ref{sec61} and Table \ref{tbl1}). 

\paragraph{\textbf{\textit{SPIS}-SEG-WL Distinguishes 3-WL Failure Cases.}} When compared with $k$-order WL tests ($k\geq 3$, SEG-WL test costs only $O(n^2)$ time complexity at each iteration, and the flexible choice of structural encoding method allows it to show a wide range of expressive capabilities. Here, we show that \textit{SPIS}-SEG-WL is able to distinguish a pair of graphs that 3-WL can not distinguish:
\begin{proposition}
There exists a pair of graphs that \textit{SPIS}-SEG-WL can distinguish, but 3-WL can not.
\label{mpro3}
\end{proposition}

\begin{proof}
Figure \ref{afig3} below shows a pair of graphs that can be distinguished by \textit{SPIS}-SEG-WL but not 3-WL. The two graphs, named as the Shrikhande graph and the Rook's $4\times 4$ graph, are both $\text{SRG}(16,6,2,2)$ and the most popular example for indistinguishability with 3-WL \citep{arvind2020weisfeiler}. To show they can be distinguished by \textit{SPIS}-SEG-WL, Figure \ref{afig3} also colors examples of SPIS between non-adjacent nodes, where the nodes at two endpoints are colored as red. In the second graph (the Shrikhande graph), one can verify that every SPIS between non-adjacent nodes has 4 nodes and 4 edges, but in the first graph (the Rook's $4\times 4$ graph) there exists SPIS between non-adjacent nodes that has 5 edges, making \textit{SPIS}-SEG-WL capable of distinguishing them.
\end{proof}

\begin{figure}[h]
\centering
\includegraphics[width=0.8\linewidth]{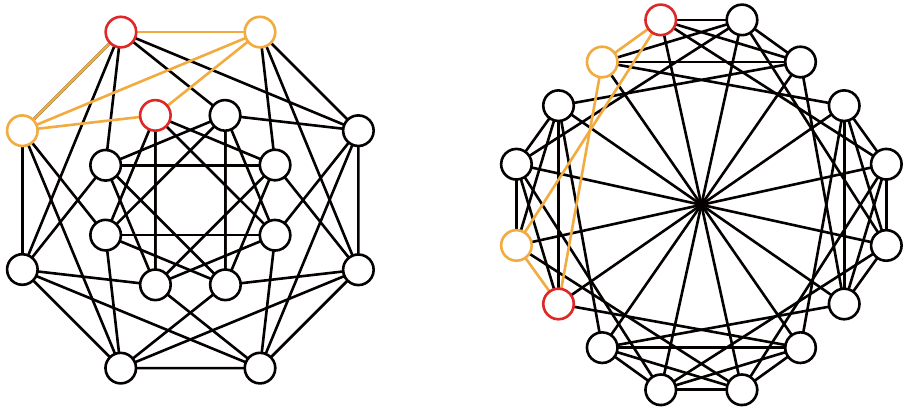}
\caption{Two graphs (the Shrikhande graph and the Rook's $4\times 4$ graph) that can be distinguished by \textit{SPIS}-SEG-WL but not 3-WL.}
\label{afig3}
\end{figure}


\paragraph{\textbf{Computing \textit{SPIS}.}} To compute \textit{SPIS} encoding on a input graph $G=(V,E)$, we first use the Floyd-Warshall algorithm \citep{floyd1962ambiguity} to compute the lengths of shortest paths between all pairs of vertices in $G$, which takes $O(n^3)$ time complexity where $n=|V|$. Next for every pair of nodes $(v,u)$, for every node $s$ we test if $s$ is in $\text{SPIS}(v,u)$ by checking if $\text{SPD}_R(v,s)+\text{SPD}_R(s,u)=\text{SPD}_R(v,u)$ holds to construct $V_{\text{SPIS}(v,u)}$, and this step also has $O(n^3)$ time complexity. Finally, for every pair of nodes $(v,u)$ we construct $E_{\text{SPIS}(v,u)}$ by computing the intersection between $V_{\text{SPIS}(v,u)}\times V_{\text{SPIS}(v,u)}$ and $E$. If we denote the average number of nodes of SPISs in the graph as $t$, then $V_{\text{SPIS}(v,u)}\times V_{\text{SPIS}(v,u)}$ can have $t^2$ edges in average and thus the final step costs $O(n^2t^2)$ complexity. The overall time complexity for computing \textit{SPIS} is then $O(n^3+n^2t^2)$. As we can reasonably expect $t^2\sim n$ on most real-world sparse graphs because SPISs should be small with respect to the entire graph, the complexity of \textit{SPIS} can be viewed as $O(n^3)$. This is quite acceptable because the time complexity for computing \textit{SPD} via Floyd-Warshall algorithm is already $O(n^3)$, and \textit{SPIS} offers a much stronger expressive power.

\begin{table*}[h]
    \begin{adjustbox}{width=1.8\columnwidth,center}
    \begin{tabular}{l|cccc|ccccccc}
        \toprule
        & \multicolumn{4}{c|}{Low-Order Graphs (Parameter: $n$)} & \multicolumn{6}{c}{Strongly Regular Graphs (Parameter: $(n,k,\lambda,\mu)$)} \\
        \midrule
        Parameter & $5$ & $6$ & $7$ & $8$ & $(25,12,5,6)$ & $(26,10,3,4)$ & $(29,14,6,7)$ & $(36,14,4,6)$ & $(40,12,2,4)$ & $(45,12,3,3)$ \\
        \midrule
        \# Graphs & 21 & 112 & 853 & 11117 & 15 & 10 & 41 & 180 & 28 & 78 \\
        \# Graph Pairs & 210 & 6216 & 363378 & 61788286 & 105 & 45 & 820 & 16110 & 378 & 3003 \\
        \midrule
        Method & \multicolumn{10}{c}{\# Indistinguishable Graph Pairs}\\
        \midrule
        WL & 0 & 3 & 17 & 312 & 105 & 45 & 820 & 16110 & 378 & 3003 \\
        \textit{SPD}-SEG-WL & 0 & 2 & 12 & 186 & 105 & 45 & 820 & 16110 & 378 & 3003 \\
        \textit{SPIS}-SEG-WL & 0 & \textbf{0} & \textbf{0} & \textbf{0} & \textbf{0} & \textbf{0} & \textbf{0} & \textbf{15} & \textbf{3} & \textbf{0} \\
        \bottomrule
    \end{tabular}
    \end{adjustbox}
    \vspace{1mm}
    \caption{Results of synthetic graph isomorphism tests.}
    \label{tbl1}
    
\end{table*}

\begin{table*}[h]
    \begin{adjustbox}{width=1.9\columnwidth,center}
    \begin{tabular}{l|cccc|cccc}
        \toprule
        Task & \multicolumn{4}{c|}{Regression} & \multicolumn{4}{c}{Classification} \\
        \midrule
        Dataset & ZINC & QM9 & QM8 & ESOL & PTC-MR & MUTAG & COX2 & PROTEINS \\
        \midrule
        Metric & \multicolumn{1}{c|}{MAE$\downarrow$} & \multicolumn{2}{c|}{Multi-MAE$\downarrow$} &  \multicolumn{1}{c|}{RMSE$\downarrow$} & \multicolumn{4}{c}{Accuracy$\uparrow$} \\
        \midrule
        \multicolumn{1}{l|}{Method} & \multicolumn{8}{c}{Results}\\
        \midrule
        GCN       & 0.469\small$\pm$0.002 & 1.006\small$\pm$0.020 & 0.0279\small$\pm$0.0001 & 0.564\small$\pm$0.015 & 67.97\small$\pm$6.49 & 85.76\small$\pm$8.75 & 80.42\small$\pm$5.23 & 76.00\small$\pm$3.20 \\
        GAT       & 0.463\small$\pm$0.002 & 1.112\small$\pm$0.018 & 0.0317\small$\pm$0.0001 & 0.552\small$\pm$0.007 & 67.21\small$\pm$2.50 & 84.59\small$\pm$6.30 & 79.36\small$\pm$7.23 & 71.15\small$\pm$7.12 \\
        GIN       & 0.408\small$\pm$0.008 & 1.225\small$\pm$0.055 & 0.0276\small$\pm$0.0001 & 0.626\small$\pm$0.017 & 68.27\small$\pm$5.11 & 89.40\small$\pm$5.40 & 82.57\small$\pm$4.55 & 75.90\small$\pm$2.80 \\
        GraphSAGE & 0.410\small$\pm$0.005 & 0.855\small$\pm$0.002 & 0.0275\small$\pm$0.0001 & 0.601\small$\pm$0.008 & 60.53\small$\pm$5.24 & 85.10\small$\pm$7.60 & 78.07\small$\pm$7.07 & 75.90\small$\pm$3.20 \\
        GSN       & 0.140\small$\pm$0.006 & - & - & - & 67.40\small$\pm$5.70 & 92.20\small$\pm$7.50 & - & 74.60\small$\pm$5.00 \\
        PNA       & 0.320\small$\pm$0.032 & - & - & - & - & - & - & - \\
        1-2-3-GNN & - & - & - & - & 60.90 & 86.10 & - & 75.50 \\
        \midrule
        MoleculeNet & - & 2.350 & 0.0150 & 0.580 & - & - & - & - \\
        \midrule
        WL & - & - & - & - & 59.90\small$\pm$4.30 & 90.40\small$\pm$5.70 & - & 75.00\small$\pm$3.10 \\
        RetGK & - & - & - & - & 62.50\small$\pm$1.60 & 90.30\small$\pm$1.10 & 80.10\small$\pm$0.90 & 76.20\small$\pm$0.50 \\
        P-WL & - & - & - & - & 64.02\small$\pm$0.82 & 90.51\small$\pm$1.34 & - & 75.31\small$\pm$0.73 \\
        FGW & - & - & - & - & 65.31\small$\pm$7.90 & 88.42\small$\pm$5.67 & 77.23\small$\pm$4.86 & 74.55\small$\pm$2.74 \\
        \midrule
        GT & 0.226\small$\pm 0.01$ & - & - & - & - & - & - & - \\
        SAN & 0.139\small$\pm 0.01$ & - & - & - & - & - & - & - \\
        SAT & 0.135 & - & - & - & - & - & - & - \\
        \midrule
        Graphormer-\textit{id} & 0.668\small$\pm$0.003 & 3.176\small$\pm$0.005 & 0.0144\small$\pm$0.0003 & 0.612\small$\pm$0.002 & 66.39\small$\pm$5.18 & 85.49\small$\pm$8.51 & 77.60\small$\pm$7.69 & 77.19\small$\pm$4.07 \\
        Graphormer-\textit{Neighbor} & 0.531\small$\pm$0.004 & 1.799\small$\pm$0.002 & 0.0141\small$\pm$0.0002 & 0.639\small$\pm$0.034 & 68.13\small$\pm$6.82 & 90.35\small$\pm$7.01 & 78.06\small$\pm$7.43 & 78.12\small$\pm$3.62 \\
        Graphormer-\textit{SPD} & 0.122\small$\pm$0.001 & 0.607\small$\pm$0.002 & 0.0079\small$\pm$0.0001 & 0.492\small$\pm$0.004 & 68.43\small$\pm$5.82 & 91.39\small$\pm$7.35 & 82.12\small$\pm$3.40 & 78.59\small$\pm$4.35 \\
        Graphormer-\textit{SPIS} & \textbf{0.115}\small$\pm$\textbf{0.001} & \textbf{0.595}\small$\pm$\textbf{0.001} & \textbf{0.0073}\small$\pm$\textbf{0.0001} & \textbf{0.484}\small$\pm$\textbf{0.005} & \textbf{69.28}\small$\pm$\textbf{5.34} & \textbf{92.48}\small$\pm$\textbf{5.87} & \textbf{83.22}\small$\pm$\textbf{2.25} & \textbf{79.41\small$\pm$1.46} \\
        \bottomrule
    \end{tabular}
    \end{adjustbox}
    \vspace{1mm}
    \caption{Results of graph representation learning benchmarks. All results except for GCN, GAT, GIN, GraphSAGE, SAT and Graphormer variants are cited from their original papers. $\downarrow$ for lower is better, and $\uparrow$ for higher is better. Appendix \ref{asec53} reports performances on QM9 by seperate tasks.}
    \label{tbl2}
    
\end{table*}

\section{Experiments}
\label{sec6}

In this section, we first perform synthetic isomorphism tests on low order graphs and strongly regular graphs to evaluate the expressive power of proposed \textit{SPIS} encoding against several previous benchmark methods. Then we show that by replacing \textit{SPD} encoding with the provably stronger \textit{SPIS}, the performance of the well-tested Graphormer model on a wide range of real-world datasets can be significantly improved.

\subsection{Synthetic Isomorphism Tests}
\label{sec61}

\paragraph{Settings.} To evaluate the structural expressive power of WL test and SEG-WL test with structural encodings described above, we first perform synthetic isomorphism tests on a collection of connected low-order graphs up to 8 nodes and strongly regular graphs up to 45 nodes\footnote{We use the database in \url{http://www.maths.gla.ac.uk/~es/srgraphs.php} to collect strongly regular graphs with the same set of parameters.}. We run the algorithms above and check how they can disambiguate non-isomorphic low order graphs with the same number of nodes and strongly regular graphs with the same parameters. The results are shown in Table \ref{tbl1}.

\paragraph{Results.} For low order graphs, results in Table \ref{tbl1} show that \textit{SPD}-SEG-WL can distinguish more non-isomorphic graphs than WL, but neither can match the effectiveness of \textit{SPIS}-SEG-WL which disambiguates any low-order graphs up to 8 nodes. As for the highly symmetric strongly regular graphs, both WL and \textit{SPD}-SEG-WL cannot discriminate any strongly regular graphs with the same parameters, yet \textit{SPIS}-SEG-WL only has few indistinguishable pairs. Compared with WL and \textit{SPD}-SEG-WL, \textit{SPIS}-SEG-WL has outstanding structural expressive power. Since many real-world graphs (like molecular graphs) consist of small motifs with highly symmetrical structures, it is reasonable to expect that graph Transformers with 
\textit{SPIS} can accurately capture significant graph structures and exhibit strong discriminative power.

\subsection{Graph Representation Learning}

\paragraph{Datasets.} To test the real-world performance of graph Transformers with proposed structural encodings, we select 8 popular graph representation learning benchmarks: 5 property regression datasets (ogb-PCQM4Mv2 \cite{hu2021ogb,hu2020open}, ZINC(subset) \citep{irwin2005zinc,dwivedi2020benchmarking}, QM9, QM8, ESOL \citep{wu2018moleculenet}) and 4 classification datasets (PTC-MR, MUTAG, COX2, PROTEINS \citep{morris2020tudataset}). Statistics of the datasets are summarized the appendix. ogb-PCQM4Mv2 is a large-scale graph regression dataset with over 3 million graphs. The ZINC dataset from benchmarking-gnn \citep{dwivedi2020benchmarking}\footnote{\url{https://github.com/graphdeeplearning/benchmarking-gnns}.} is a subset of the ZINC chemical database \citep{irwin2005zinc} with 12000 molecules, and the task is to predict the solubility of molecules. We follow the guidelines and use the predefined split for training, validation and testing. QM9 and QM8 \citep{ruddigkeit2012enumeration,ramakrishnan2014quantum,ramakrishnan2015electronic} are two molecular datasets containing small organic molecules up to 9 and 8 heavy atoms, and the task is to predict molecular properties calculated with ab initio Density Functional Theory (DFT). We follow the guidelines in MoleculeNet \citep{wu2018moleculenet} for choosing regression tasks and metrics. We perform joint training on 12 tasks for QM9 and 16 tasks for QM8. ESOL is also a molecular regression dataset in MoleculeNet containing water solubility data for compounds\footnote{QM8, QM9 and ESOL are available at \url{http://moleculenet.ai/datasets-1} (MIT 2.0 license).}. PTC-MR, MUTAG, COX2 and PROTEINS are four graph classification datasets collected from TUDataset \citep{morris2020tudataset}\footnote{The four datasets are available at \url{https://chrsmrrs.github.io/datasets/}.}. On graph regression datasets, We use random 8:1:1 split for training, validation, and testing except for ZINC, and report the performance averaged over 3 runs. On graph classification datasets, we use 10-fold cross validation with 90\% training and 10\% testing, and report the mean best accuracy.

\paragraph{Settings and Baselines.} To investigate how the expressive power of structural encodings affects the benchmark performance of real graph Transformers, we first choose the Graphormer \citep{ying2021transformers} as the backbone model for testing structural encoding since Graphormer proposes SPD, which we have characterized and has expressivity stronger than WL, and the way Graphormer introduces relative structural encodings can correspond to our Theorem \ref{thm2} which analyzes a simple Transformer network incorporating relative encodings via attention biases. The original Graphormer utilizes a $\textit{SPD}_R$ relative structural encoding (discussed in Appendix \ref{asec4}), so we name it as Graphormer-\textit{SPD}. We build a new Graphormer-\textit{SPIS} model by replacing the $\textit{SPD}_R$ encoding with $\textit{SPIS}_R$ encoding as an improved version of Graphormer while keeping other network components unchanged. Similarly, we use Graphormer-\textit{id} and Graphormer-\textit{Neighbor} as less expressive Graphormer variants. We also include the GraphGPS \cite{rampasek2022GPS} model and its variants GraphGPS-\textit{SPD} and GraphGPS-\textit{SPIS} into comparison on the large-scale ogbn-PCQM4Mv2 dataset with over 3 million graphs. The Transformer module in the basic GraphGPS model does not incorporate structural encoding, thus it can be considered as including \textit{id} structural encoding. We construct versions of the GraphGPS model incorporating \textit{Neighbor}, \textit{SPD}, and \textit{SPIS} structural encodings via attention biases to validate the impact of structural encoding expressivity on performance for large-scale graph tasks.

In addition, we compare the above Graphormer variants against (i) GNNs including GCN \citep{kipf2016semi}, GIN \citep{xu2018powerful}, GAT \citep{velivckovic2017graph}, GraphSAGE \citep{hamilton2018inductive}, GSN \citep{bouritsas2022improving}, PNA \citep{corso2020principal} and 1-2-3-GNN \citep{morris2019weisfeiler}; (ii) best performances collected by MoleculeNet paper \citep{wu2018moleculenet}; (iii) graph kernel based methods including WL subtree kernel \citep{shervashidze2011weisfeiler}, RetGK \citep{zhang2018retgk}, P-WL \citep{rieck2019persistent} and FGW \citep{titouan2019optimal}; (iv) graph Transformers including GT \citep{dwivedi2020generalization}, SAN \citep{kreuzer2021rethinking}, SAT \citep{chen2022structure}, GRPE \cite{park2022grpe} and EGT \cite{hussain2022global}. One can find the detailed descriptions of Graphormer variants, baselines, and training settings in Appendix \ref{asec52}.

\begin{table}[h]
    \centering
    \begin{tabular}{l|cc}
    \toprule
    Model       & Training MAE & Validation MAE \\
    \midrule
    GCN & n/a & 0.1379 \\
    GCN-virtual & n/a & 0.1153 \\
    GIN & n/a & 0.1195 \\
    GIN-virtual & n/a & 0.1083 \\
    \midrule
    GRPE & n/a & 0.0890 \\
    EGT & n/a & 0.0869 \\
    \midrule
    Graphormer (-\textit{SPD}) & 0.0348 & 0.0864 \\
    Graphormer-\textit{SPIS} & 0.0350 & 0.0861 \\
    \midrule
    GraphGPS (medium) (-\textit{id}) & 0.0726 & 0.0858 \\
    GraphGPS-\textit{Neighbor} & 0.0730 & 0.0856 \\
    GraphGPS-\textit{SPD} & 0.0719 & 0.0853 \\
    GraphGPS-\textit{SPIS} & 0.0710 & \textbf{0.0850} \\
    \bottomrule
    \end{tabular}
    \vspace{4mm}
    \caption{Results on ogb-PCQM4Mv2 dataset.}
    \label{tbl3}
\end{table}

\paragraph{Results.} Table \ref{tbl2} and \ref{tbl3} presents the results of graph representation learning benchmarks. It can be observed that the performances of Graphormer variants mostly align with their relative ranking of expressive power (\textit{SPIS} $\succeq$ \textit{SPD} $\succeq$ \textit{Neighbor} $\succeq$ \textit{id}), and replacing the \textit{SPD} encoding in Graphormer with the proposed stronger \textit{SPIS} encoding results in a consistent performance improvement, demonstrating that real-world performance of graph Transformers can benefit from theoretically expressive structural encoding designs. Equipped with the provably powerful \textit{SPIS} encoding, Graphormer-\textit{SPIS} achieves state-of-the-art performance and outperforms existing graph Transformers and GNNs, which echoes our theoretical results on the strong expressive power of \textit{SPIS}. Meanwhile, when employing the less expressive \textit{Neighbor} and \textit{id} encoding, the Transformer network loses the ability to accurately distinguish graph structures, leading to a significant performance drop. In the case of the GraphGPS model, the experimental results follow the same pattern. The original model achieved a certain level of performance improvement after incorporating the structural encoding in the Transformer layer. The stronger the expression ability of the structural encoding, the more significant the performance improvement. Overall, the experimental results demonstrate that our theoretical analysis has a practical impact on enhancing the performance of graph Transformers in various graph tasks.

\section{Conclusion}
\label{sec7}
In this paper, we introduce SEG-WL test as a novel unified framework for analyzing the expressive power of graph Transformers. In this framework, we theoretically characterize how to improve the expressivity of graph Transformers with respect to WL test and GNNs, and propose a provably powerful structural encoding method \textit{SPIS}. Experiments have verified that the performances of benchmark graph Transformers can benefit from this theory-oriented extension. We also discuss our work's limitations and potential social impact in Appendix \ref{asec6}.


\bibliography{lib}


\begin{thebibliography}{53}


\ifx \showCODEN    \undefined \def \showCODEN     #1{\unskip}     \fi
\ifx \showDOI      \undefined \def \showDOI       #1{#1}\fi
\ifx \showISBNx    \undefined \def \showISBNx     #1{\unskip}     \fi
\ifx \showISBNxiii \undefined \def \showISBNxiii  #1{\unskip}     \fi
\ifx \showISSN     \undefined \def \showISSN      #1{\unskip}     \fi
\ifx \showLCCN     \undefined \def \showLCCN      #1{\unskip}     \fi
\ifx \shownote     \undefined \def \shownote      #1{#1}          \fi
\ifx \showarticletitle \undefined \def \showarticletitle #1{#1}   \fi
\ifx \showURL      \undefined \def \showURL       {\relax}        \fi
\providecommand\bibfield[2]{#2}
\providecommand\bibinfo[2]{#2}
\providecommand\natexlab[1]{#1}
\providecommand\showeprint[2][]{arXiv:#2}

\bibitem[\protect\citeauthoryear{Arvind, Fuhlbr{\"u}ck, K{\"o}bler, and
  Verbitsky}{Arvind et~al\mbox{.}}{2020}]%
        {arvind2020weisfeiler}
\bibfield{author}{\bibinfo{person}{Vikraman Arvind}, \bibinfo{person}{Frank
  Fuhlbr{\"u}ck}, \bibinfo{person}{Johannes K{\"o}bler}, {and}
  \bibinfo{person}{Oleg Verbitsky}.} \bibinfo{year}{2020}\natexlab{}.
\newblock \showarticletitle{On Weisfeiler-Leman invariance: Subgraph counts and
  related graph properties}.
\newblock \bibinfo{journal}{\emph{J. Comput. System Sci.}}
  \bibinfo{volume}{113} (\bibinfo{year}{2020}), \bibinfo{pages}{42--59}.
\newblock


\bibitem[\protect\citeauthoryear{Azizian and Lelarge}{Azizian and
  Lelarge}{2020}]%
        {azizian2020expressive}
\bibfield{author}{\bibinfo{person}{Waiss Azizian} {and} \bibinfo{person}{Marc
  Lelarge}.} \bibinfo{year}{2020}\natexlab{}.
\newblock \showarticletitle{Expressive power of invariant and equivariant graph
  neural networks}.
\newblock \bibinfo{journal}{\emph{arXiv preprint arXiv:2006.15646}}
  (\bibinfo{year}{2020}).
\newblock


\bibitem[\protect\citeauthoryear{Ba, Kiros, and Hinton}{Ba
  et~al\mbox{.}}{2016}]%
        {ba2016layer}
\bibfield{author}{\bibinfo{person}{Jimmy~Lei Ba}, \bibinfo{person}{Jamie~Ryan
  Kiros}, {and} \bibinfo{person}{Geoffrey~E Hinton}.}
  \bibinfo{year}{2016}\natexlab{}.
\newblock \showarticletitle{Layer normalization}.
\newblock \bibinfo{journal}{\emph{arXiv preprint arXiv:1607.06450}}
  (\bibinfo{year}{2016}).
\newblock


\bibitem[\protect\citeauthoryear{Bouritsas, Frasca, Zafeiriou, and
  Bronstein}{Bouritsas et~al\mbox{.}}{2022}]%
        {bouritsas2022improving}
\bibfield{author}{\bibinfo{person}{Giorgos Bouritsas},
  \bibinfo{person}{Fabrizio Frasca}, \bibinfo{person}{Stefanos~P Zafeiriou},
  {and} \bibinfo{person}{Michael Bronstein}.} \bibinfo{year}{2022}\natexlab{}.
\newblock \showarticletitle{Improving graph neural network expressivity via
  subgraph isomorphism counting}.
\newblock \bibinfo{journal}{\emph{IEEE Transactions on Pattern Analysis and
  Machine Intelligence}} (\bibinfo{year}{2022}).
\newblock


\bibitem[\protect\citeauthoryear{Chen, O’Bray, and Borgwardt}{Chen
  et~al\mbox{.}}{2022}]%
        {chen2022structure}
\bibfield{author}{\bibinfo{person}{Dexiong Chen}, \bibinfo{person}{Leslie
  O’Bray}, {and} \bibinfo{person}{Karsten Borgwardt}.}
  \bibinfo{year}{2022}\natexlab{}.
\newblock \showarticletitle{Structure-aware transformer for graph
  representation learning}. In \bibinfo{booktitle}{\emph{International
  Conference on Machine Learning}}. PMLR, \bibinfo{pages}{3469--3489}.
\newblock


\bibitem[\protect\citeauthoryear{Corso, Cavalleri, Beaini, Li{\`o}, and
  Veli{\v{c}}kovi{\'c}}{Corso et~al\mbox{.}}{2020}]%
        {corso2020principal}
\bibfield{author}{\bibinfo{person}{Gabriele Corso}, \bibinfo{person}{Luca
  Cavalleri}, \bibinfo{person}{Dominique Beaini}, \bibinfo{person}{Pietro
  Li{\`o}}, {and} \bibinfo{person}{Petar Veli{\v{c}}kovi{\'c}}.}
  \bibinfo{year}{2020}\natexlab{}.
\newblock \showarticletitle{Principal neighbourhood aggregation for graph
  nets}.
\newblock \bibinfo{journal}{\emph{Advances in Neural Information Processing
  Systems}}  \bibinfo{volume}{33} (\bibinfo{year}{2020}),
  \bibinfo{pages}{13260--13271}.
\newblock


\bibitem[\protect\citeauthoryear{Devlin, Chang, Lee, and Toutanova}{Devlin
  et~al\mbox{.}}{2018}]%
        {devlin2018bert}
\bibfield{author}{\bibinfo{person}{Jacob Devlin}, \bibinfo{person}{Ming-Wei
  Chang}, \bibinfo{person}{Kenton Lee}, {and} \bibinfo{person}{Kristina
  Toutanova}.} \bibinfo{year}{2018}\natexlab{}.
\newblock \showarticletitle{Bert: Pre-training of deep bidirectional
  transformers for language understanding}.
\newblock \bibinfo{journal}{\emph{arXiv preprint arXiv:1810.04805}}
  (\bibinfo{year}{2018}).
\newblock


\bibitem[\protect\citeauthoryear{Dwivedi and Bresson}{Dwivedi and
  Bresson}{2020}]%
        {dwivedi2020generalization}
\bibfield{author}{\bibinfo{person}{Vijay~Prakash Dwivedi} {and}
  \bibinfo{person}{Xavier Bresson}.} \bibinfo{year}{2020}\natexlab{}.
\newblock \showarticletitle{A generalization of transformer networks to
  graphs}.
\newblock \bibinfo{journal}{\emph{arXiv preprint arXiv:2012.09699}}
  (\bibinfo{year}{2020}).
\newblock


\bibitem[\protect\citeauthoryear{Dwivedi, Joshi, Laurent, Bengio, and
  Bresson}{Dwivedi et~al\mbox{.}}{2020}]%
        {dwivedi2020benchmarking}
\bibfield{author}{\bibinfo{person}{Vijay~Prakash Dwivedi},
  \bibinfo{person}{Chaitanya~K Joshi}, \bibinfo{person}{Thomas Laurent},
  \bibinfo{person}{Yoshua Bengio}, {and} \bibinfo{person}{Xavier Bresson}.}
  \bibinfo{year}{2020}\natexlab{}.
\newblock \showarticletitle{Benchmarking graph neural networks}.
\newblock \bibinfo{journal}{\emph{arXiv preprint arXiv:2003.00982}}
  (\bibinfo{year}{2020}).
\newblock


\bibitem[\protect\citeauthoryear{Floyd}{Floyd}{1962}]%
        {floyd1962ambiguity}
\bibfield{author}{\bibinfo{person}{Robert~W Floyd}.}
  \bibinfo{year}{1962}\natexlab{}.
\newblock \showarticletitle{On ambiguity in phrase structure languages}.
\newblock \bibinfo{journal}{\emph{Commun. ACM}} \bibinfo{volume}{5},
  \bibinfo{number}{10} (\bibinfo{year}{1962}), \bibinfo{pages}{526}.
\newblock


\bibitem[\protect\citeauthoryear{Freeman}{Freeman}{1977}]%
        {freeman1977set}
\bibfield{author}{\bibinfo{person}{Linton~C Freeman}.}
  \bibinfo{year}{1977}\natexlab{}.
\newblock \showarticletitle{A set of measures of centrality based on
  betweenness}.
\newblock \bibinfo{journal}{\emph{Sociometry}} (\bibinfo{year}{1977}),
  \bibinfo{pages}{35--41}.
\newblock


\bibitem[\protect\citeauthoryear{Grohe}{Grohe}{2017}]%
        {grohe2017descriptive}
\bibfield{author}{\bibinfo{person}{Martin Grohe}.}
  \bibinfo{year}{2017}\natexlab{}.
\newblock \bibinfo{booktitle}{\emph{Descriptive complexity, canonisation, and
  definable graph structure theory}}. Vol.~\bibinfo{volume}{47}.
\newblock \bibinfo{publisher}{Cambridge University Press}.
\newblock


\bibitem[\protect\citeauthoryear{Hamilton, Ying, and Leskovec}{Hamilton
  et~al\mbox{.}}{2018}]%
        {hamilton2018inductive}
\bibfield{author}{\bibinfo{person}{William~L. Hamilton}, \bibinfo{person}{Rex
  Ying}, {and} \bibinfo{person}{Jure Leskovec}.}
  \bibinfo{year}{2018}\natexlab{}.
\newblock \bibinfo{title}{Inductive Representation Learning on Large Graphs}.
\newblock
\newblock
\showeprint[arxiv]{1706.02216}~[cs.SI]


\bibitem[\protect\citeauthoryear{Hornik, Stinchcombe, and White}{Hornik
  et~al\mbox{.}}{1989}]%
        {hornik1989multilayer}
\bibfield{author}{\bibinfo{person}{Kurt Hornik}, \bibinfo{person}{Maxwell
  Stinchcombe}, {and} \bibinfo{person}{Halbert White}.}
  \bibinfo{year}{1989}\natexlab{}.
\newblock \showarticletitle{Multilayer feedforward networks are universal
  approximators}.
\newblock \bibinfo{journal}{\emph{Neural networks}} \bibinfo{volume}{2},
  \bibinfo{number}{5} (\bibinfo{year}{1989}), \bibinfo{pages}{359--366}.
\newblock


\bibitem[\protect\citeauthoryear{Hu, Fey, Ren, Nakata, Dong, and Leskovec}{Hu
  et~al\mbox{.}}{2021}]%
        {hu2021ogb}
\bibfield{author}{\bibinfo{person}{Weihua Hu}, \bibinfo{person}{Matthias Fey},
  \bibinfo{person}{Hongyu Ren}, \bibinfo{person}{Maho Nakata},
  \bibinfo{person}{Yuxiao Dong}, {and} \bibinfo{person}{Jure Leskovec}.}
  \bibinfo{year}{2021}\natexlab{}.
\newblock \showarticletitle{Ogb-lsc: A large-scale challenge for machine
  learning on graphs}.
\newblock \bibinfo{journal}{\emph{arXiv preprint arXiv:2103.09430}}
  (\bibinfo{year}{2021}).
\newblock


\bibitem[\protect\citeauthoryear{Hu, Fey, Zitnik, Dong, Ren, Liu, Catasta, and
  Leskovec}{Hu et~al\mbox{.}}{2020}]%
        {hu2020open}
\bibfield{author}{\bibinfo{person}{Weihua Hu}, \bibinfo{person}{Matthias Fey},
  \bibinfo{person}{Marinka Zitnik}, \bibinfo{person}{Yuxiao Dong},
  \bibinfo{person}{Hongyu Ren}, \bibinfo{person}{Bowen Liu},
  \bibinfo{person}{Michele Catasta}, {and} \bibinfo{person}{Jure Leskovec}.}
  \bibinfo{year}{2020}\natexlab{}.
\newblock \showarticletitle{Open graph benchmark: Datasets for machine learning
  on graphs}.
\newblock \bibinfo{journal}{\emph{arXiv preprint arXiv:2005.00687}}
  (\bibinfo{year}{2020}).
\newblock


\bibitem[\protect\citeauthoryear{Hussain, Zaki, and Subramanian}{Hussain
  et~al\mbox{.}}{2022}]%
        {hussain2022global}
\bibfield{author}{\bibinfo{person}{Md~Shamim Hussain},
  \bibinfo{person}{Mohammed~J Zaki}, {and} \bibinfo{person}{Dharmashankar
  Subramanian}.} \bibinfo{year}{2022}\natexlab{}.
\newblock \showarticletitle{Global self-attention as a replacement for graph
  convolution}. In \bibinfo{booktitle}{\emph{Proceedings of the 28th ACM SIGKDD
  Conference on Knowledge Discovery and Data Mining}}.
  \bibinfo{pages}{655--665}.
\newblock


\bibitem[\protect\citeauthoryear{Irwin and Shoichet}{Irwin and
  Shoichet}{2005}]%
        {irwin2005zinc}
\bibfield{author}{\bibinfo{person}{John~J Irwin} {and} \bibinfo{person}{Brian~K
  Shoichet}.} \bibinfo{year}{2005}\natexlab{}.
\newblock \showarticletitle{ZINC- a free database of commercially available
  compounds for virtual screening}.
\newblock \bibinfo{journal}{\emph{Journal of chemical information and
  modeling}} \bibinfo{volume}{45}, \bibinfo{number}{1} (\bibinfo{year}{2005}),
  \bibinfo{pages}{177--182}.
\newblock


\bibitem[\protect\citeauthoryear{Kipf and Welling}{Kipf and Welling}{2016}]%
        {kipf2016semi}
\bibfield{author}{\bibinfo{person}{Thomas~N Kipf} {and} \bibinfo{person}{Max
  Welling}.} \bibinfo{year}{2016}\natexlab{}.
\newblock \showarticletitle{Semi-supervised classification with graph
  convolutional networks}.
\newblock \bibinfo{journal}{\emph{arXiv preprint arXiv:1609.02907}}
  (\bibinfo{year}{2016}).
\newblock


\bibitem[\protect\citeauthoryear{Kreuzer, Beaini, Hamilton, L{\'e}tourneau, and
  Tossou}{Kreuzer et~al\mbox{.}}{2021}]%
        {kreuzer2021rethinking}
\bibfield{author}{\bibinfo{person}{Devin Kreuzer}, \bibinfo{person}{Dominique
  Beaini}, \bibinfo{person}{William~L Hamilton}, \bibinfo{person}{Vincent
  L{\'e}tourneau}, {and} \bibinfo{person}{Prudencio Tossou}.}
  \bibinfo{year}{2021}\natexlab{}.
\newblock \showarticletitle{Rethinking Graph Transformers with Spectral
  Attention}.
\newblock \bibinfo{journal}{\emph{arXiv preprint arXiv:2106.03893}}
  (\bibinfo{year}{2021}).
\newblock


\bibitem[\protect\citeauthoryear{Liu, Pan, He, Song, Jiang, and Shang}{Liu
  et~al\mbox{.}}{2020}]%
        {liu2020neural}
\bibfield{author}{\bibinfo{person}{Xin Liu}, \bibinfo{person}{Haojie Pan},
  \bibinfo{person}{Mutian He}, \bibinfo{person}{Yangqiu Song},
  \bibinfo{person}{Xin Jiang}, {and} \bibinfo{person}{Lifeng Shang}.}
  \bibinfo{year}{2020}\natexlab{}.
\newblock \showarticletitle{Neural subgraph isomorphism counting}. In
  \bibinfo{booktitle}{\emph{Proceedings of the 26th ACM SIGKDD International
  Conference on Knowledge Discovery \& Data Mining}}.
  \bibinfo{pages}{1959--1969}.
\newblock


\bibitem[\protect\citeauthoryear{Loshchilov and Hutter}{Loshchilov and
  Hutter}{2016}]%
        {loshchilov2016sgdr}
\bibfield{author}{\bibinfo{person}{Ilya Loshchilov} {and}
  \bibinfo{person}{Frank Hutter}.} \bibinfo{year}{2016}\natexlab{}.
\newblock \showarticletitle{Sgdr: Stochastic gradient descent with warm
  restarts}.
\newblock \bibinfo{journal}{\emph{arXiv preprint arXiv:1608.03983}}
  (\bibinfo{year}{2016}).
\newblock


\bibitem[\protect\citeauthoryear{Loshchilov and Hutter}{Loshchilov and
  Hutter}{2018}]%
        {loshchilov2018decoupled}
\bibfield{author}{\bibinfo{person}{Ilya Loshchilov} {and}
  \bibinfo{person}{Frank Hutter}.} \bibinfo{year}{2018}\natexlab{}.
\newblock \showarticletitle{Decoupled Weight Decay Regularization}. In
  \bibinfo{booktitle}{\emph{International Conference on Learning
  Representations}}.
\newblock


\bibitem[\protect\citeauthoryear{Maron, Ben-Hamu, Serviansky, and Lipman}{Maron
  et~al\mbox{.}}{2019}]%
        {maron2019provably}
\bibfield{author}{\bibinfo{person}{Haggai Maron}, \bibinfo{person}{Heli
  Ben-Hamu}, \bibinfo{person}{Hadar Serviansky}, {and} \bibinfo{person}{Yaron
  Lipman}.} \bibinfo{year}{2019}\natexlab{}.
\newblock \showarticletitle{Provably powerful graph networks}.
\newblock \bibinfo{journal}{\emph{Advances in neural information processing
  systems}}  \bibinfo{volume}{32} (\bibinfo{year}{2019}).
\newblock


\bibitem[\protect\citeauthoryear{Monti, Otness, and Bronstein}{Monti
  et~al\mbox{.}}{2018}]%
        {monti2018motifnet}
\bibfield{author}{\bibinfo{person}{Federico Monti}, \bibinfo{person}{Karl
  Otness}, {and} \bibinfo{person}{Michael~M Bronstein}.}
  \bibinfo{year}{2018}\natexlab{}.
\newblock \showarticletitle{Motifnet: a motif-based graph convolutional network
  for directed graphs}. In \bibinfo{booktitle}{\emph{2018 IEEE Data Science
  Workshop (DSW)}}. IEEE, \bibinfo{pages}{225--228}.
\newblock


\bibitem[\protect\citeauthoryear{Morris, Kriege, Bause, Kersting, Mutzel, and
  Neumann}{Morris et~al\mbox{.}}{2020a}]%
        {morris2020tudataset}
\bibfield{author}{\bibinfo{person}{Christopher Morris}, \bibinfo{person}{Nils~M
  Kriege}, \bibinfo{person}{Franka Bause}, \bibinfo{person}{Kristian Kersting},
  \bibinfo{person}{Petra Mutzel}, {and} \bibinfo{person}{Marion Neumann}.}
  \bibinfo{year}{2020}\natexlab{a}.
\newblock \showarticletitle{Tudataset: A collection of benchmark datasets for
  learning with graphs}.
\newblock \bibinfo{journal}{\emph{arXiv preprint arXiv:2007.08663}}
  (\bibinfo{year}{2020}).
\newblock


\bibitem[\protect\citeauthoryear{Morris, Rattan, and Mutzel}{Morris
  et~al\mbox{.}}{2020b}]%
        {morris2020weisfeiler}
\bibfield{author}{\bibinfo{person}{Christopher Morris}, \bibinfo{person}{Gaurav
  Rattan}, {and} \bibinfo{person}{Petra Mutzel}.}
  \bibinfo{year}{2020}\natexlab{b}.
\newblock \showarticletitle{Weisfeiler and Leman go sparse: Towards scalable
  higher-order graph embeddings}.
\newblock \bibinfo{journal}{\emph{Advances in Neural Information Processing
  Systems}}  \bibinfo{volume}{33} (\bibinfo{year}{2020}),
  \bibinfo{pages}{21824--21840}.
\newblock


\bibitem[\protect\citeauthoryear{Morris, Ritzert, Fey, Hamilton, Lenssen,
  Rattan, and Grohe}{Morris et~al\mbox{.}}{2019}]%
        {morris2019weisfeiler}
\bibfield{author}{\bibinfo{person}{Christopher Morris}, \bibinfo{person}{Martin
  Ritzert}, \bibinfo{person}{Matthias Fey}, \bibinfo{person}{William~L
  Hamilton}, \bibinfo{person}{Jan~Eric Lenssen}, \bibinfo{person}{Gaurav
  Rattan}, {and} \bibinfo{person}{Martin Grohe}.}
  \bibinfo{year}{2019}\natexlab{}.
\newblock \showarticletitle{Weisfeiler and leman go neural: Higher-order graph
  neural networks}. In \bibinfo{booktitle}{\emph{Proceedings of the AAAI
  conference on artificial intelligence}}, Vol.~\bibinfo{volume}{33}.
  \bibinfo{pages}{4602--4609}.
\newblock


\bibitem[\protect\citeauthoryear{Park, Chang, Lee, Kim, et~al\mbox{.}}{Park
  et~al\mbox{.}}{2022}]%
        {park2022grpe}
\bibfield{author}{\bibinfo{person}{Wonpyo Park}, \bibinfo{person}{Woong-Gi
  Chang}, \bibinfo{person}{Donggeon Lee}, \bibinfo{person}{Juntae Kim},
  {et~al\mbox{.}}} \bibinfo{year}{2022}\natexlab{}.
\newblock \showarticletitle{Grpe: Relative positional encoding for graph
  transformer}. In \bibinfo{booktitle}{\emph{ICLR2022 Machine Learning for Drug
  Discovery}}.
\newblock


\bibitem[\protect\citeauthoryear{Ramakrishnan, Dral, Rupp, and
  Von~Lilienfeld}{Ramakrishnan et~al\mbox{.}}{2014}]%
        {ramakrishnan2014quantum}
\bibfield{author}{\bibinfo{person}{Raghunathan Ramakrishnan},
  \bibinfo{person}{Pavlo~O Dral}, \bibinfo{person}{Matthias Rupp}, {and}
  \bibinfo{person}{O~Anatole Von~Lilienfeld}.} \bibinfo{year}{2014}\natexlab{}.
\newblock \showarticletitle{Quantum chemistry structures and properties of 134
  kilo molecules}.
\newblock \bibinfo{journal}{\emph{Scientific data}} \bibinfo{volume}{1},
  \bibinfo{number}{1} (\bibinfo{year}{2014}), \bibinfo{pages}{1--7}.
\newblock


\bibitem[\protect\citeauthoryear{Ramakrishnan, Hartmann, Tapavicza, and
  Von~Lilienfeld}{Ramakrishnan et~al\mbox{.}}{2015}]%
        {ramakrishnan2015electronic}
\bibfield{author}{\bibinfo{person}{Raghunathan Ramakrishnan},
  \bibinfo{person}{Mia Hartmann}, \bibinfo{person}{Enrico Tapavicza}, {and}
  \bibinfo{person}{O~Anatole Von~Lilienfeld}.} \bibinfo{year}{2015}\natexlab{}.
\newblock \showarticletitle{Electronic spectra from TDDFT and machine learning
  in chemical space}.
\newblock \bibinfo{journal}{\emph{The Journal of chemical physics}}
  \bibinfo{volume}{143}, \bibinfo{number}{8} (\bibinfo{year}{2015}),
  \bibinfo{pages}{084111}.
\newblock


\bibitem[\protect\citeauthoryear{Ramp\'{a}\v{s}ek, Galkin, Dwivedi, Luu, Wolf,
  and Beaini}{Ramp\'{a}\v{s}ek et~al\mbox{.}}{2022}]%
        {rampasek2022GPS}
\bibfield{author}{\bibinfo{person}{Ladislav Ramp\'{a}\v{s}ek},
  \bibinfo{person}{Mikhail Galkin}, \bibinfo{person}{Vijay~Prakash Dwivedi},
  \bibinfo{person}{Anh~Tuan Luu}, \bibinfo{person}{Guy Wolf}, {and}
  \bibinfo{person}{Dominique Beaini}.} \bibinfo{year}{2022}\natexlab{}.
\newblock \showarticletitle{{Recipe for a General, Powerful, Scalable Graph
  Transformer}}.
\newblock \bibinfo{journal}{\emph{arXiv:2205.12454}} (\bibinfo{year}{2022}).
\newblock


\bibitem[\protect\citeauthoryear{Rieck, Bock, and Borgwardt}{Rieck
  et~al\mbox{.}}{2019}]%
        {rieck2019persistent}
\bibfield{author}{\bibinfo{person}{Bastian Rieck}, \bibinfo{person}{Christian
  Bock}, {and} \bibinfo{person}{Karsten Borgwardt}.}
  \bibinfo{year}{2019}\natexlab{}.
\newblock \showarticletitle{A persistent weisfeiler-lehman procedure for graph
  classification}. In \bibinfo{booktitle}{\emph{International Conference on
  Machine Learning}}. PMLR, \bibinfo{pages}{5448--5458}.
\newblock


\bibitem[\protect\citeauthoryear{Rong, Bian, Xu, Xie, Wei, Huang, and
  Huang}{Rong et~al\mbox{.}}{2020}]%
        {rong2020self}
\bibfield{author}{\bibinfo{person}{Yu Rong}, \bibinfo{person}{Yatao Bian},
  \bibinfo{person}{Tingyang Xu}, \bibinfo{person}{Weiyang Xie},
  \bibinfo{person}{Ying Wei}, \bibinfo{person}{Wenbing Huang}, {and}
  \bibinfo{person}{Junzhou Huang}.} \bibinfo{year}{2020}\natexlab{}.
\newblock \showarticletitle{Self-supervised graph transformer on large-scale
  molecular data}.
\newblock \bibinfo{journal}{\emph{arXiv preprint arXiv:2007.02835}}
  (\bibinfo{year}{2020}).
\newblock


\bibitem[\protect\citeauthoryear{Ruddigkeit, Van~Deursen, Blum, and
  Reymond}{Ruddigkeit et~al\mbox{.}}{2012}]%
        {ruddigkeit2012enumeration}
\bibfield{author}{\bibinfo{person}{Lars Ruddigkeit}, \bibinfo{person}{Ruud
  Van~Deursen}, \bibinfo{person}{Lorenz~C Blum}, {and}
  \bibinfo{person}{Jean-Louis Reymond}.} \bibinfo{year}{2012}\natexlab{}.
\newblock \showarticletitle{Enumeration of 166 billion organic small molecules
  in the chemical universe database GDB-17}.
\newblock \bibinfo{journal}{\emph{Journal of chemical information and
  modeling}} \bibinfo{volume}{52}, \bibinfo{number}{11} (\bibinfo{year}{2012}),
  \bibinfo{pages}{2864--2875}.
\newblock


\bibitem[\protect\citeauthoryear{Sato, Yamada, and Kashima}{Sato
  et~al\mbox{.}}{2021}]%
        {sato2021random}
\bibfield{author}{\bibinfo{person}{Ryoma Sato}, \bibinfo{person}{Makoto
  Yamada}, {and} \bibinfo{person}{Hisashi Kashima}.}
  \bibinfo{year}{2021}\natexlab{}.
\newblock \showarticletitle{Random features strengthen graph neural networks}.
  In \bibinfo{booktitle}{\emph{Proceedings of the 2021 SIAM International
  Conference on Data Mining (SDM)}}. SIAM, \bibinfo{pages}{333--341}.
\newblock


\bibitem[\protect\citeauthoryear{Shervashidze, Schweitzer, Van~Leeuwen,
  Mehlhorn, and Borgwardt}{Shervashidze et~al\mbox{.}}{2011}]%
        {shervashidze2011weisfeiler}
\bibfield{author}{\bibinfo{person}{Nino Shervashidze}, \bibinfo{person}{Pascal
  Schweitzer}, \bibinfo{person}{Erik~Jan Van~Leeuwen}, \bibinfo{person}{Kurt
  Mehlhorn}, {and} \bibinfo{person}{Karsten~M Borgwardt}.}
  \bibinfo{year}{2011}\natexlab{}.
\newblock \showarticletitle{Weisfeiler-lehman graph kernels.}
\newblock \bibinfo{journal}{\emph{Journal of Machine Learning Research}}
  \bibinfo{volume}{12}, \bibinfo{number}{9} (\bibinfo{year}{2011}).
\newblock


\bibitem[\protect\citeauthoryear{Titouan, Courty, Tavenard, and
  Flamary}{Titouan et~al\mbox{.}}{2019}]%
        {titouan2019optimal}
\bibfield{author}{\bibinfo{person}{Vayer Titouan}, \bibinfo{person}{Nicolas
  Courty}, \bibinfo{person}{Romain Tavenard}, {and} \bibinfo{person}{R{\'e}mi
  Flamary}.} \bibinfo{year}{2019}\natexlab{}.
\newblock \showarticletitle{Optimal transport for structured data with
  application on graphs}. In \bibinfo{booktitle}{\emph{International Conference
  on Machine Learning}}. PMLR, \bibinfo{pages}{6275--6284}.
\newblock


\bibitem[\protect\citeauthoryear{Vaswani, Shazeer, Parmar, Uszkoreit, Jones,
  Gomez, Kaiser, and Polosukhin}{Vaswani et~al\mbox{.}}{2017}]%
        {vaswani2017attention}
\bibfield{author}{\bibinfo{person}{Ashish Vaswani}, \bibinfo{person}{Noam
  Shazeer}, \bibinfo{person}{Niki Parmar}, \bibinfo{person}{Jakob Uszkoreit},
  \bibinfo{person}{Llion Jones}, \bibinfo{person}{Aidan~N Gomez},
  \bibinfo{person}{{\L}ukasz Kaiser}, {and} \bibinfo{person}{Illia
  Polosukhin}.} \bibinfo{year}{2017}\natexlab{}.
\newblock \showarticletitle{Attention is all you need}. In
  \bibinfo{booktitle}{\emph{Advances in neural information processing
  systems}}. \bibinfo{pages}{5998--6008}.
\newblock


\bibitem[\protect\citeauthoryear{Veli{\v{c}}kovi{\'c}, Cucurull, Casanova,
  Romero, Lio, and Bengio}{Veli{\v{c}}kovi{\'c} et~al\mbox{.}}{2017}]%
        {velivckovic2017graph}
\bibfield{author}{\bibinfo{person}{Petar Veli{\v{c}}kovi{\'c}},
  \bibinfo{person}{Guillem Cucurull}, \bibinfo{person}{Arantxa Casanova},
  \bibinfo{person}{Adriana Romero}, \bibinfo{person}{Pietro Lio}, {and}
  \bibinfo{person}{Yoshua Bengio}.} \bibinfo{year}{2017}\natexlab{}.
\newblock \showarticletitle{Graph attention networks}.
\newblock \bibinfo{journal}{\emph{arXiv preprint arXiv:1710.10903}}
  (\bibinfo{year}{2017}).
\newblock


\bibitem[\protect\citeauthoryear{Vignac, Loukas, and Frossard}{Vignac
  et~al\mbox{.}}{2020}]%
        {vignac2020building}
\bibfield{author}{\bibinfo{person}{Clement Vignac}, \bibinfo{person}{Andreas
  Loukas}, {and} \bibinfo{person}{Pascal Frossard}.}
  \bibinfo{year}{2020}\natexlab{}.
\newblock \showarticletitle{Building powerful and equivariant graph neural
  networks with structural message-passing}.
\newblock \bibinfo{journal}{\emph{Advances in Neural Information Processing
  Systems}}  \bibinfo{volume}{33} (\bibinfo{year}{2020}),
  \bibinfo{pages}{14143--14155}.
\newblock


\bibitem[\protect\citeauthoryear{Weisfeiler and Leman}{Weisfeiler and
  Leman}{1968}]%
        {weisfeiler1968reduction}
\bibfield{author}{\bibinfo{person}{Boris Weisfeiler} {and}
  \bibinfo{person}{Andrei Leman}.} \bibinfo{year}{1968}\natexlab{}.
\newblock \showarticletitle{The reduction of a graph to canonical form and the
  algebra which appears therein}.
\newblock \bibinfo{journal}{\emph{NTI, Series}} \bibinfo{volume}{2},
  \bibinfo{number}{9} (\bibinfo{year}{1968}), \bibinfo{pages}{12--16}.
\newblock


\bibitem[\protect\citeauthoryear{Wijesinghe and Wang}{Wijesinghe and
  Wang}{2021}]%
        {wijesinghe2021new}
\bibfield{author}{\bibinfo{person}{Asiri Wijesinghe} {and}
  \bibinfo{person}{Qing Wang}.} \bibinfo{year}{2021}\natexlab{}.
\newblock \showarticletitle{A New Perspective on" How Graph Neural Networks Go
  Beyond Weisfeiler-Lehman?"}. In \bibinfo{booktitle}{\emph{International
  Conference on Learning Representations}}.
\newblock


\bibitem[\protect\citeauthoryear{Wu, Zhao, Li, Wipf, and Yan}{Wu
  et~al\mbox{.}}{[n.\,d.]}]%
        {wunodeformer}
\bibfield{author}{\bibinfo{person}{Qitian Wu}, \bibinfo{person}{Wentao Zhao},
  \bibinfo{person}{Zenan Li}, \bibinfo{person}{David Wipf}, {and}
  \bibinfo{person}{Junchi Yan}.} \bibinfo{year}{[n.\,d.]}\natexlab{}.
\newblock \showarticletitle{NodeFormer: A Scalable Graph Structure Learning
  Transformer for Node Classification}. In \bibinfo{booktitle}{\emph{Advances
  in Neural Information Processing Systems}}.
\newblock


\bibitem[\protect\citeauthoryear{Wu, Ramsundar, Feinberg, Gomes, Geniesse,
  Pappu, Leswing, and Pande}{Wu et~al\mbox{.}}{2018}]%
        {wu2018moleculenet}
\bibfield{author}{\bibinfo{person}{Zhenqin Wu}, \bibinfo{person}{Bharath
  Ramsundar}, \bibinfo{person}{Evan~N Feinberg}, \bibinfo{person}{Joseph
  Gomes}, \bibinfo{person}{Caleb Geniesse}, \bibinfo{person}{Aneesh~S Pappu},
  \bibinfo{person}{Karl Leswing}, {and} \bibinfo{person}{Vijay Pande}.}
  \bibinfo{year}{2018}\natexlab{}.
\newblock \showarticletitle{MoleculeNet: a benchmark for molecular machine
  learning}.
\newblock \bibinfo{journal}{\emph{Chemical science}} \bibinfo{volume}{9},
  \bibinfo{number}{2} (\bibinfo{year}{2018}), \bibinfo{pages}{513--530}.
\newblock


\bibitem[\protect\citeauthoryear{Xiong, Yang, He, Zheng, Zheng, Xing, Zhang,
  Lan, Wang, and Liu}{Xiong et~al\mbox{.}}{2020}]%
        {xiong2020layer}
\bibfield{author}{\bibinfo{person}{Ruibin Xiong}, \bibinfo{person}{Yunchang
  Yang}, \bibinfo{person}{Di He}, \bibinfo{person}{Kai Zheng},
  \bibinfo{person}{Shuxin Zheng}, \bibinfo{person}{Chen Xing},
  \bibinfo{person}{Huishuai Zhang}, \bibinfo{person}{Yanyan Lan},
  \bibinfo{person}{Liwei Wang}, {and} \bibinfo{person}{Tieyan Liu}.}
  \bibinfo{year}{2020}\natexlab{}.
\newblock \showarticletitle{On layer normalization in the transformer
  architecture}. In \bibinfo{booktitle}{\emph{International Conference on
  Machine Learning}}. PMLR, \bibinfo{pages}{10524--10533}.
\newblock


\bibitem[\protect\citeauthoryear{Xu, Hu, Leskovec, and Jegelka}{Xu
  et~al\mbox{.}}{2018}]%
        {xu2018powerful}
\bibfield{author}{\bibinfo{person}{Keyulu Xu}, \bibinfo{person}{Weihua Hu},
  \bibinfo{person}{Jure Leskovec}, {and} \bibinfo{person}{Stefanie Jegelka}.}
  \bibinfo{year}{2018}\natexlab{}.
\newblock \showarticletitle{How powerful are graph neural networks?}
\newblock \bibinfo{journal}{\emph{arXiv preprint arXiv:1810.00826}}
  (\bibinfo{year}{2018}).
\newblock


\bibitem[\protect\citeauthoryear{Ying, Cai, Luo, Zheng, Ke, He, Shen, and
  Liu}{Ying et~al\mbox{.}}{2021}]%
        {ying2021transformers}
\bibfield{author}{\bibinfo{person}{Chengxuan Ying}, \bibinfo{person}{Tianle
  Cai}, \bibinfo{person}{Shengjie Luo}, \bibinfo{person}{Shuxin Zheng},
  \bibinfo{person}{Guolin Ke}, \bibinfo{person}{Di He},
  \bibinfo{person}{Yanming Shen}, {and} \bibinfo{person}{Tie-Yan Liu}.}
  \bibinfo{year}{2021}\natexlab{}.
\newblock \showarticletitle{Do Transformers Really Perform Bad for Graph
  Representation?}
\newblock \bibinfo{journal}{\emph{arXiv preprint arXiv:2106.05234}}
  (\bibinfo{year}{2021}).
\newblock


\bibitem[\protect\citeauthoryear{You, Gomes-Selman, Ying, and Leskovec}{You
  et~al\mbox{.}}{2021}]%
        {you2021identity}
\bibfield{author}{\bibinfo{person}{Jiaxuan You}, \bibinfo{person}{Jonathan
  Gomes-Selman}, \bibinfo{person}{Rex Ying}, {and} \bibinfo{person}{Jure
  Leskovec}.} \bibinfo{year}{2021}\natexlab{}.
\newblock \showarticletitle{Identity-aware graph neural networks}.
\newblock \bibinfo{journal}{\emph{arXiv preprint arXiv:2101.10320}}
  (\bibinfo{year}{2021}).
\newblock


\bibitem[\protect\citeauthoryear{Yun, Bhojanapalli, Rawat, Reddi, and
  Kumar}{Yun et~al\mbox{.}}{2019}]%
        {yun2019transformers}
\bibfield{author}{\bibinfo{person}{Chulhee Yun}, \bibinfo{person}{Srinadh
  Bhojanapalli}, \bibinfo{person}{Ankit~Singh Rawat},
  \bibinfo{person}{Sashank~J Reddi}, {and} \bibinfo{person}{Sanjiv Kumar}.}
  \bibinfo{year}{2019}\natexlab{}.
\newblock \showarticletitle{Are transformers universal approximators of
  sequence-to-sequence functions?}
\newblock \bibinfo{journal}{\emph{arXiv preprint arXiv:1912.10077}}
  (\bibinfo{year}{2019}).
\newblock


\bibitem[\protect\citeauthoryear{Zhang, Liu, Hu, and Lee}{Zhang
  et~al\mbox{.}}{2022}]%
        {zhang2022hierarchical}
\bibfield{author}{\bibinfo{person}{Zaixi Zhang}, \bibinfo{person}{Qi Liu},
  \bibinfo{person}{Qingyong Hu}, {and} \bibinfo{person}{Chee-Kong Lee}.}
  \bibinfo{year}{2022}\natexlab{}.
\newblock \showarticletitle{Hierarchical Graph Transformer with Adaptive Node
  Sampling}.
\newblock \bibinfo{journal}{\emph{arXiv preprint arXiv:2210.03930}}
  (\bibinfo{year}{2022}).
\newblock


\bibitem[\protect\citeauthoryear{Zhang, Wang, Xiang, Huang, and Nehorai}{Zhang
  et~al\mbox{.}}{2018}]%
        {zhang2018retgk}
\bibfield{author}{\bibinfo{person}{Zhen Zhang}, \bibinfo{person}{Mianzhi Wang},
  \bibinfo{person}{Yijian Xiang}, \bibinfo{person}{Yan Huang}, {and}
  \bibinfo{person}{Arye Nehorai}.} \bibinfo{year}{2018}\natexlab{}.
\newblock \showarticletitle{Retgk: Graph kernels based on return probabilities
  of random walks}.
\newblock \bibinfo{journal}{\emph{Advances in Neural Information Processing
  Systems}}  \bibinfo{volume}{31} (\bibinfo{year}{2018}).
\newblock


\bibitem[\protect\citeauthoryear{Zhao, Li, Wen, Wang, Liu, Sun, Xie, and
  Ye}{Zhao et~al\mbox{.}}{2021}]%
        {zhao2021gophormer}
\bibfield{author}{\bibinfo{person}{Jianan Zhao}, \bibinfo{person}{Chaozhuo Li},
  \bibinfo{person}{Qianlong Wen}, \bibinfo{person}{Yiqi Wang},
  \bibinfo{person}{Yuming Liu}, \bibinfo{person}{Hao Sun},
  \bibinfo{person}{Xing Xie}, {and} \bibinfo{person}{Yanfang Ye}.}
  \bibinfo{year}{2021}\natexlab{}.
\newblock \showarticletitle{Gophormer: Ego-Graph Transformer for Node
  Classification}.
\newblock \bibinfo{journal}{\emph{arXiv preprint arXiv:2110.13094}}
  (\bibinfo{year}{2021}).
\newblock


\end{thebibliography}
\bibliographystyle{ACM-Reference-Format}

\appendix

\titlespacing*{\section}{0pt}{0.7\baselineskip}{0.7\baselineskip}
\titlespacing*{\subsection}{0pt}{0.7\baselineskip}{0.7\baselineskip}
\titlespacing*{\paragraph}{0pt}{0.5\baselineskip}{0.5\baselineskip}

\clearpage

\section{Proofs}
\label{asec1}
\subsection{Theorem \ref{thm1}}
\label{asec11}

We first restate Theorem \ref{thm1} in a more generalized version which can be applied to both cases when the graph embedding is computed by a global readout function or virtual node trick:

\begin{customthm}{1}
For any structural encoding scheme $S=(f_A,f_R)$ and labeled graph $G=(V,E)$ with label map $h_0:V\to\gX$, if a graph neural model $\gA:\gG\to\R^d$ satisfies the following conditions:
\begin{enumerate}
\itemsep0em
    \item $\gA$ computes the initial node embeddings with
    \begin{align}
        l_0(v)=\phi(h_0(v),f_A(v,G)),
    \end{align}
    \item $\gA$ aggregates and updates node embeddings iteratively with
    \begin{align}
        l_t(v)=\sigma(\ldblbrace(l_{t-1}(u),f_R(v,u,G)):u\in V\rdblbrace),
    \end{align}
    where $\phi$ and $\sigma$ above are model-specific functions,
    \item The final graph embedding is computed by a global readout on the multiset of node features $\ldblbrace l_t(v):v\in V\rdblbrace$, or represented by the embedding of node $s$ such that for any $u,v\in V$, $f_R(s,v,G)=f_R(s,u,G)=f_R(v,s,G)=f_R(u,s,G)$.
\end{enumerate}
then for any labeled graphs $G_1$ and $G_2$, if $\gA$ maps them to different embeddings, $S$-SEG-WL also decides $G_1$ and $G_2$ are not isomorphic.
\end{customthm}

\begin{proof}
We first show that for any node $v,u$ at iteration $t$, if $S$-SEG-WL generates $g_t(v)=g_t(u)$, then $\gA$ also generates the same embeddings for $v$ and $u$ as $l_t(v)=l_t(u)$. For $t=0$ this proposition holds because if $g_0(v)=g_0(u)$ then $v$ and $u$ must have the same input label and absolute structural encoding, which leads to $l_0(v)=l_0(u)$. Suppose this proposition holds for iteration $0,1,\ldots,t$ and $g_{t+1}(v)=g_{t+1}(u)$. From the injectiveness of function $\Phi$, we have 
\begin{align}
    \ldblbrace (g_t(r),f_R(v,r,G)):r\in V_v\rdblbrace=\ldblbrace (g_t(r),f_R(u,r,G)):r\in V_u\rdblbrace,
\end{align}
where $V_v$ is the node set of graph that $v$ belongs to, which is the same for $V_u$. 
If two finite multisets are identical, then the elements in the two multisets can be matched in pairs. Therefore, according to our assumption at iteration $t$ such that $g_t(v)=g_t(u)\implies l_t(v)=l_t(u)$, we have
\begin{align}
    \ldblbrace (l_t(r),f_R(v,r,G)):r\in V\rdblbrace=\ldblbrace (l_t(r),f_R(u,r,G)):r\in V\rdblbrace.
\end{align}
Considering $\gA$ updates node labels by $l_{t+1}(v)=\sigma(\ldblbrace(l_t(r),f_R(v,r,G)):r\in V\rdblbrace)$, $l_{t+1}(v)=l_{t+1}(u)$ holds. This proves the proposition above by induction. Now that for any iteration $t$ we have $g_t(v)=g_t(u)\implies l_t(v)=l_t(u)$, indicating that a mapping $\psi_t$ exists such that for any node $v$, $l_t(v)=\psi_t(g_t(v))$.

Now consider two graphs $G_1$ and $G_2$ where $\gA$ maps them to different embeddings after $t$ iterations. If $\gA$ computes the graph embedding by a readout function on the multiset of node features, then $\ldblbrace l_t(r):r\in V\rdblbrace$ must be different for two graphs. Since $\ldblbrace l_t(r):r\in V\rdblbrace=\ldblbrace \psi_i(g_t(r)):r\in V\rdblbrace$, $\ldblbrace g_t(r):r\in V\rdblbrace$ must also be different for two graphs, which shows that $S$-SEG-WL decides $G_1$ and $G_2$ are not isomorphic. Meanwhile, if the graph embedding is represented by embedding of node $s$ such that for any $u,v\in V$, $f_R(s,v,G)=f_R(s,u,G)=f_R(v,s,G)=f_R(u,s,G)$, then $l_t(s)$ is different for two graphs. Since $l_t(s)$ is generated by $l_t(s)=\sigma(\ldblbrace (l_{t-1}(r),f_R(s,r,G)):r\in V\rdblbrace)$ and $f_R(s,r,G)$ is the same for every $r\in V$, $\ldblbrace l_{t-1}(r):r\in V\rdblbrace$ must be different for two graphs, which goes back to the situation we have discussed above. Therefore, the proof is completed.
\end{proof}

\subsection{Theorem \ref{thm2}}
\label{asec12}

Our proof for Theorem \ref{thm2} is largely based on the proof for the universal approximation theorem of the Transformer architecture, so it is strongly recommended to go through the proof in \cite{yun2019transformers} before reading our proof in the next section.

\subsubsection{bias-GT Model}
To present a simple and flexible example on building theoretically powerful graph Transformers, we propose bias-GT, a graph Transformer model that works under any structural encoding schemes with minimal modifications to the original Transformer architecture. More concretely, for $S=(f_A,f_R)$ and input graph $G$, the input embedding of node $v$ is computed by
\begin{align}
    l_0(v)=\text{Linear}(\text{Concat}(h_0(v),f_A(v,G)),
\end{align}
where Linear$(\cdot)$ is a linear layer, Concat$(\cdot)$ refers to the concatenation operation. At every Transformer layer, the relative structural encodings are introduced as transformed attention biases. For every node pair $(u,v)$, the final attention weight $a_{uv}$ from node $u$ to $v$ is computed by
\begin{align}
    a_{uv}=\bar a_{uv}+\text{Embedding}(f_R(u,v,G)),
\end{align}
where $\bar a_{uv}$ is the original attention weight computed by scaled-dot self-attention, and $\text{Embedding}(\cdot)$ transforms relative embeddings in $\gC$ to $\R$ using via embedding lookup or linear layers. All remaining network components stay the same with the original Transformer architecture. This bias-GT model offers a straightforward strategy for injecting strutural information to the Transformer and can be viewed as a simplified version of some exisiting models \citep{ying2021transformers,zhao2021gophormer}. We will use $S$-bias-GT to denote bias-GT network with structural encoding scheme $S$. The proposition below shows that $S$-SEG-WL test limits the expressive power of $S$-bias-GT:

\begin{proposition}
For any regular structural encoding scheme $S=(f_A,f_R)$ and two graphs $G_1,G_2$, if $S$-bias-GT maps them to different embeddings, $S$-SEG-WL also decides $G_1$ and $G_2$ are not isomorphic.
\label{apro0}
\end{proposition}
\begin{proof}

We only need to check the conditions in Theorem \ref{thm1}. For the first condition, $S$-bias-GT computes the initial node embeddings with
\begin{align}
    l_0(v)=\phi(h_0(v),f_A(v,G))=\text{Linear}(\text{Concat}(\bar h_0(v),f_A(v,G)),
\end{align}
and for the second condition, since $S$ is regular, the relative structural encoding functions satisfy $f_R(v,v,G)\neq f_R(v,u,G)$ for $v,u\in V$, then a function operated on $\ldblbrace (l_{t-1}(u),f_R(v,u,G)):u\in V\rdblbrace$ can be viewed as a function operated on $(h_v,\ldblbrace (l_{t-1}(u),f_R(v,u,G)):u\in V\rdblbrace)$ because $f_R(v,v,G)$ is different from all other relative encodings. $S$-bias-GT updatesthe node embeddings with
\begin{align}
    l_t(v)= &\sigma(h_v,\ldblbrace (l_{t-1}(u),f_R(v,u,G)):u\in V\rdblbrace)\\
    =&\text{FFN}({\text{Concat}}_{i=1,\ldots,h}({\sum_{u\in V}}w_{vu}^i l_{t-1}(v)W^i_Q)W_O),\\
    &\text{where }w_{vu}^i=\frac{\text{exp}(\bar\alpha_{vu}^i)}{\sum_{r\in V}\text{exp}(\bar\alpha_{vr}^i)}\\
    &\text{and }\bar\alpha_{vu}^i=\frac{(l_{t-1}(v)W_Q^i) (l_{t-1}(u)W_K^i)^\top}{\sqrt{d}}\\
    &\quad +\text{Embedding}_i(f_R(v,u,G)).
\end{align}
$W_Q^i,W_K^i,W_V^i,W_O$ above are projection matrices, $\text{FFN}$ is the feed-forward layer,
and layer normalization and residual connections are omitted for clarity. The function $\sigma$ is basically the computation steps of the Transformer with $f_R(v,u,G)$ injected as attention bias. Since the graph embedding can be computed by a global readout function, according to Theorem \ref{thm1}, the proof is completed.
\end{proof}

\subsubsection{Explainations on Theorem \ref{thm2}}
When the input graph order $n$ is fixed, let the input be $G=(V,E)$ with label map $h_0$ and $V=\{v_1,\ldots,v_n\}$. To properly define this approximation process, for some structural encoding scheme $S$, the input for both SEG-WL test and bias-GT network $g$ is viewed as the feature matrix $\mX_0=[h_0(v_1),\ldots,h_0(v_n)]\in\R^{d\times n}$ and the adjacency matrix $\mA$ with permutation invariance. We define the SEG-WL test function $f_t$ by stacking all labels generated by $t$-iteration SEG-WL test in $f_t(\mX_0,\mA)=[g_t(v_1),\ldots,g_t(v_n)]$, and the output of $g$ is similarly defined by stacking all feature vectors generated by the network. We define the $\mathsf{d}_p$ distance between $f_t$ and $g$ as $\mathsf{d}_p(f_t,g)=\max_{\mA}\mathsf{d}_p(f_t(\cdot,\mA),g(\cdot,\mA))$, where $\mathsf{d}_p(f_t(\cdot,\mA),g(\cdot,\mA))$ is the $\ell^p$ distance on $\R^{d\times n}$ between $f_t$ and $g$ when $\mA$ is fixed. Following \citep{yun2019transformers}, $\mathsf{d}_p$ also stands for $\ell^p$ distance between functions in the remaining context.

$\Phi(\ldblbrace (g_{t-1}(v_i),f_R(v_j,v_i,G)):v_i\in V)\rdblbrace)$ can be viewed as a permutation (of node order) invariant function $\Phi(\mX_{t-1},\mW_j)$, where $\mX_{t-1}=[g_{t-1}(v_1),\ldots,g_{t-1}(v_n)]\in\R^{d\times n}$ is the matrix of node labels and $\mW_j=[f_R(v_j,v_1,G)),\ldots,f_R(v_j,v_n,G))]\in\gC^n$. The assumption that $\Phi$ can be extended to a continuous function with respect to node labels means that, for any fixed $\mW_j$, $\Phi(\mX_{t-1},\mW_j)$ is a continuous function with respect to any entry-wise $\ell^p$ norm of $\mX_{t-1}$ with compact support in $\R^{d\times n}$ (since $\gX$ is compact).
\subsubsection{Proof for Theorem \ref{thm2}}
\begin{proof}
Since the first iteration of SEG-WL test can be arbitrarily approximated by performing a linear layer on embeddings generated by concatnating the initial embeddings and absolute positional encodings, according to the universal approximation theorem \citep{hornik1989multilayer} and Lipschitz continuity of feed-forward layers, the key technical challenge in proving Theorem \ref{thm2} is showing that each iteration $i$ in $1,\ldots,t$ of SEG-WL test can be approximated arbitrarily well using the bias-GT network. Let $f$ stands for one iteration of $S$-SEG-WL test, with input and output defined according to $f_t$. We denote $\bm{\alpha}_{i,j}=f_R(v_i,v_j,G)$ for simplicity, and let $\vx_i$ be the input labels of $v_i$ for $f$. Then $f$ can be viewed as:
\begin{align}
    f(\mX,\mA)=[\Phi(\ldblbrace(\vx_i,\bm{\alpha}_{1,i})\rdblbrace_{i=1,\ldots,n}),\ldots,\Phi(\ldblbrace(\vx_i,\bm{\alpha}_{n,i})\rdblbrace_{i=1,\ldots,n})].\label{eq1}
\end{align}
That is, if our Transformer network is capable of approximating the multiset function $\Phi$ that takes the feature matrix and structural encodings as input, then it can approximate $f$ at any precision because the output of $f$ contains $n$ entries computed individually by $\Phi$. As we have mentioned, $\Phi$ can be rewritten to the following equivalent form:
\begin{align}
    &\Phi(\ldblbrace(\vx_i,\bm{\alpha}_{j,i})\rdblbrace_{i=1,\ldots,n})=\Phi(\mX,\mW_j),\text{ where }\mX=[\vx_1,\ldots,\vx_n] \\
    &\text{ and }\mW_j=[\bm{\alpha}_{j,1},\ldots,\bm{\alpha}_{j,n}].
\end{align}
In this form, $\Phi$ is permutation equivariant such that for any permutation matrix $\mP$, $\Phi(\mX\mP,\mW_j\mP)=\Phi(\mX,\mW_j)$.

The major problem is that the bias-GT network $g$ only take $\mX$ as feature input while incorporating structural encodings $\mW_j$ in attention layers as biases. According to our assumptions, we can assume without generality that $\gX\subset(0,1)^d$ and the compact support of extented function $\Phi$ with respect to $\mX$ is contained within $[0,1]^{d\times n}$. We follow the proof structure outlined in \cite{yun2019transformers}.

\paragraph{Step 1: Approximate $f$ by $\bar f$, a piece-wise constant function with respect to $\mX$.} According to previous assumptions and statements in \cite{yun2019transformers}, for any fixed $\mW_j$, $\Phi$ is a uniform continuous function (because $\Phi$ has compact support) with respect to the argument $\mX$.
Suppose for $\mW_j$, there exists $\delta_{\mW_j}$ such that for any $\mX,\mY$,$\|\mX-\mY\|_\infty<\delta_{\mW_j}$ we have $\|\Phi(\mX,\mW_j)-\Phi(\mY,\mW_j)\|_p<\frac{\epsilon}3$. Since the possible graph structures of order $n$ is finite, $\gC$ is a finite set and the possible choices of $\mW_j$ is also finite. Therefore, we can pick $\delta=\min_{\mW_j}\{\delta_{\mW_j}\}$, then for any $\mX,\mY,\mW_j$, if $\|\mX-\mY\|_\infty<\delta$ we have $\|\Phi(\mX,\mW_j)-\Phi(\mY,\mW_j)\|_p<\frac{\epsilon}3$. Accordingly, we can define a piece-wise constant function $\bar\Phi$ to approximate $\Phi$ as

\begin{align}
    \bar \Phi(\mX,\mW_j)=\sum_{\mL\in\sG_{\delta}}\Phi(\mC_{\mL},\mW_j)\mathbbm{1}\{\mX\in\sS_{\mL}\},
\end{align}
where $\sS_{\mL}$ is a cube of width $\delta$ with $\mL$ being one of its vertices, $\mC_{\mL}\in\sS_{\mL}$ is the center point of $\sS_{\mL}$ (Please refer to Appendix B.1 of \citep{yun2019transformers} for a detailed explanation). By the uniform continuity of $\Phi$, we can prove $\|\Phi(\mX,\mW_j)-\bar\Phi(\mX,\mW_j)\|_p<\frac{\epsilon}3$ for any $\mX,\mW_j$. Also, it is trivial to verify that $\bar \Phi$ is permutation equivariant. By defining $\bar f$ by replacing function $\Phi$ with $\bar \Phi$ in Equation \ref{eq1}, we have $\mathsf{d}_p(f,\bar f)\leq\frac{\epsilon}3.$

\paragraph{Step 2: Approximate $\bar f$ with modified bias-GT network.} In this step we aim to approximate $\bar f$ using a modified bias-GT network, where the softmax operator $\sigma[\cdot]$ and $\text{ReLU}(\cdot)$ are replaced by the $\gamma$-hardmax operator $\sigma_{\text{H},\gamma}[\cdot]$ and an activation finction $\phi$ that is a piece-wise linear function with at most three pieces in which at least one piece is constant. Note that the $\gamma$-hardmax operator is defined by adding $\gamma>0$ to non-zero elements of $\sigma_{\text{H}}$.

\begin{proposition}
$\bar\Phi$ can be approximated by a modified bias-GT network $\bar g$ such that $\mathsf{d}_p(\bar f,\bar g)\leq\frac{\epsilon}3.$ \label{apro1}
\end{proposition}

\paragraph{Step 3: Approximate modified bias-GT network with (original) bias-GT network.} Finally, we will show that the modified bias-GT $\bar g$ can be approximated by the original bias-GT architecture.
\begin{proposition}
$\bar g$ can be approximated by a bias-GT network $g$ such that $\mathsf{d}_p(g,\bar g)\leq\frac{\epsilon}3.$ \label{apro2}
\end{proposition}
Following \citep{yun2019transformers}, along with three steps above, we prove that a single $S$-SEG-WL iteration $f$ can be arbitrarily approximated with a bias-GT network $g$. By stacking such bias-GT networks, we show that $S$-SEG-WL with any number of iterations can be approximated by $S$-bias-GT at any precision. We next provide proofs for the two propositions.
\end{proof}

\subsubsection{Proof for Proposition \ref{apro2}}
\label{P2}
\begin{proof}
We only need to notice that for any $\mA$,
$\sigma_{\text{H},\gamma}(\mA)\to \sigma_{\text{H}}(\mA)$ as $\gamma\to 0$. Then together with Appendix B.2 of \cite{yun2019transformers}, we can finish the proof.
\end{proof}

\subsubsection{Proof for Proposition \ref{apro1}}
\begin{proof}
We will prove this statement in five major steps:
\begin{enumerate}
    \item Given input $\mX$, a group of feed-forward layers in the modified Transformer network can quantize $\mX$ to an element $\mL$ on the grid $\sG_{\delta}:=\{0,\delta,\ldots,1-\delta\}^{d\times n}.$
    \item A group of additional feed-forward layers then scales $\mL$ to a different level, where for every $l_j:=\vu^{\top}\mL_{:,j}, $ $l_j\in\{1,\delta^{-1},\delta^{-2},\ldots,\delta^{-\delta^{-d}+1}\}$ holds. ($\vu=(1,\delta^{-1},\delta^{-2},\ldots,\delta^{-d+1})$.) 
    \item A group of biased self-attention layers perform global shift on $\mL$, such that for any $i$ and $j$, the shifted $l_i$ and $l_j$ are different if and only if their corresponding multisets of label-RSE tuples ($\ldblbrace(\vx_k,\bm{\alpha}_{i,k}):k=1,\ldots,n\rdblbrace$ for $l_i$) are different.
    \item Next, a group of self-attention layers map the shifted $\mL$ to the desirable \textbf{contextual mappings} $q(\mL)$. (defined in \cite{yun2019transformers})
    \item Finally, a group of feed-forward layers can map elements of the contextual embeddings $q(\mL)$ to the desirable values in the piece-wise constant function.
\end{enumerate}
Smiliar to Section 4 of \cite{yun2019transformers}, Proposition \ref{apro1} can be proved with five steps above, where the major difference here is in Step 1-3 we create \textbf{contextual mappings} for both node features and relative structural encodings. Next we explain the five steps in detail.

\paragraph{Step 1.} Since $\gX$ is bounded, we can assume without generality that $\gX\subset (0,1)^d$. Thus, according to Lemma 5 in \cite{yun2019transformers}, the input $\mX$ can be quantized to grid $\sG_{\delta}:=\{0,\delta,\ldots,1-\delta\}^{d\times n}.$ We still use $\vx_i$ to denote the quantized feature vector.

\paragraph{Step 2.} Before this step, we have $l_j\in[0:\delta:\delta^{-d+1}-\delta]$. Our goal in this step is to scale each $l_j$ to $\delta^{-\delta^{-1}l_j}$. For every entry $\mL_{:,j}$ in $\mL$, the scaling function is defined as
\begin{align}
    \mL_{:,j}\mapsto\mL_{:,j}+(\delta^{-\delta^{-1}\vu^\top\mL_{:,j}}-\vu^\top\mL_{:,j})\ve^{(1)},
\end{align}
We use a group of feed-forward layers to approximate this function, which is possible because Transformer has residual connections. Note that after this process, $\{1,\delta^{-1},\delta^{-2},\ldots,\delta^{-\delta^{-d}+1}\}$ contains all possible values for $l_j$. As our proof can have $\delta$ arbitrarily small, we assume $\delta^{-1}>n$. 

\paragraph{Step 3.} Since $\gC$ is finite we may assume $\gC=\{1,2,\ldots,c\}$, and let $\mW=\{\bm{\alpha}_{r,s}\}_{r,s=1,\ldots,n}$. We use one self-attention layer consists of $c$ attention heads to perform the desired global shift. We first define
\begin{align}
    &\phi_i(\mW) = \{\phi_i(\bm{\alpha}_{r,s})\}_{r,s=1,\ldots,n},\\
    &\text{where }
    \phi_i(x)=\begin{cases}
    1, \text{ if }x=i,\\
    0, \text{ else.}
    \end{cases}
\end{align}
Then, for $i=1,2,\ldots,c$, the $i$-th attention head is defined as
\begin{align}
    &\psi_i(\mZ)=\ve^{(1)}\vu^\top\mZ\sigma_{\text{H},\delta^{-p+1}/n!}(\phi_i(\mW))
    \label{eq_psi1}
\end{align}
where $p=-\delta^{-d}$. Noticing $\lim_{\delta\to 0}\delta^{-p+1}/n!=0$ and the fact in \ref{P2}, the selected $\sigma_{\text{H},\delta^{-p+1}/n!}$ is acceptable. This $\phi_i(x)$ function can be learned by embedding layers operated on the relative structural encodings. And the final attention layer is computed as
\begin{align}
    \Psi(\mZ)=\mZ+\sum_{i=1}^cn!\delta^{(3p+q)i}\psi_i(\mZ),
    \label{eq_psi2}
\end{align}
where $q$ satisfies $\delta^{q+1}\leq n!\leq\delta^q$. Note that $p$ and $q$ are both negative. For the convenience of further description, we define $\vu^\top\Psi(\mL)=[\bar l_1,\ldots,\bar l_n].$

\paragraph{Explanation on Step 2 and 3.} We aim to generate the \textbf{bijective column id mapping} for each $\ldblbrace(\vx_j,\bm{\alpha}_{i,j}):j=1,\ldots,n\rdblbrace$, while using only $\mX$ as feature input and the structural encodings are leveraged by shift operations in Step 2 and 3. We further prove this in Proposition \ref{apro1} below:
\begin{proposition}
For any $u,v\in\{1,\ldots,n\}$, $\bar l_u=\bar l_v$ if and only if $\ldblbrace(\vx_j,\bm{\alpha}_{u,j}):j=1,\ldots,n\rdblbrace=\ldblbrace(\vx_j,\bm{\alpha}_{v,j}):j=1,\ldots,n\rdblbrace$, and every $\bar l_u$ is bounded.
\label{apro1}
\end{proposition}
\begin{proof}
For each node $u$, we define $Y(u,i)=\ldblbrace\vx_v:\bm{\alpha}_{u,v}=i\rdblbrace$ and $S(u,i)=\sum_{\bm{\alpha}_{u,v}=i}l_v$, where $l_v$ is the scaled $\vu^\top\mL_{:,v}$ after Step 2.

Let the first row of $\psi_i(\mZ)$ be $[r_i(1),r_i(2),\ldots,r_i(n)].$ We first show that $r_i(u)=r_i(v)$ if and only if $Y(u,i)=Y(v,i)$. Due to the ingenious construction of $\vu$ in \cite{yun2019transformers}, $l_j$ has been an injective descriptor of $\vx_j$ before the scaling in Step 2. Since the scaling in Step 2 is injective, the scaled $l_j$ also becomes an injective descriptor of $\vx_j$. According to our scaling strategy, the scaled $l_j$ can be viewed as a $p$-digit one-hot representation of $\vx_j$. Noticing $\delta^{-1}>n$, $S(u,j)$, as the summation of these scaled $l_j$, also becomes a unique descriptor of $Y(u,i)$ and $1\leq S(u,j)\leq\delta^p$.

\begin{definition}
Suppose the set of possible values of $a$ is $P$. Then for any $u,v\in P$, if $|u-v|$ is always an integer multiple of $s$, then we call $s$ the minimal distance between any unique choices of $a$.
\end{definition}
Accordingly, the minimal distance between any unique choices of $S(u,j)$ is $1$ because the the minimal distance between any scaled $l_j$ is $1$.

Next, before discussing $\psi_i$ in Equation (\ref{eq_psi1}), (\ref{eq_psi2}) and $r_i$, we first present a lamma:

\begin{lemma}
For real numbers $a,b$, the minimal distance between any unique choices of $b$ is $s$, and $a\leq m$ holds. $a+b$ becomes a unique descriptor of $a$ if $2m<s$.
\label{lemma1}
\end{lemma}
\begin{proof}
Suppose we have $a_1+b_1=a_2+b_2$ and $a_1\neq a_2$. Then we have
\begin{align}
    |a_1-a_2|=|b_1-b_2|,
\end{align}
and $|b_1-b_2|\geq s$, $|a_1-a_2|\leq 2m$. Then the proof is completed by contradiction.
\end{proof}

Given the definition of $\psi_i$ (please refer to Appendix B.5 in \citep{yun2019transformers} for more details on the selective shift operation, which is the basis for the construction of $\psi_i$), we have 
\begin{align}
    r_i(u)=(\frac 1 k+\frac{\delta^{-p+1}}{n!})S(u,i),
\end{align}
where $k=|Y(u,i)|$. According to the range of scaled $l_j$, we have
\begin{align}
    1\leq \frac 1 k S(u,i)\leq\delta^{p+1},\\
    \frac{\delta^{-p+1}}{n!}\leq\frac{\delta^{-p+1}}{n!}S(u,i)\leq \frac{\delta}{n!}.
\end{align}

It is easy to infer that the minimal distance between any unique choices of $\frac 1 k S(u,i)$ is an integer multiple of $\frac{1}{n!}$, and we have
\begin{align}
    2\cdot\frac{\delta}{n!}<2\cdot\frac{1}{2n!}=\frac{1}{n!}.
\end{align}
According to Lemma \ref{lemma1}, $r_i(u)$ is a unique descriptor of $\frac{\delta^{-p+1}}{n!}S(u,i)$, then it is also a unique descriptor of $Y(u,j)$. Now we have $1\leq r_i(u)\leq \delta^p$ and the minimal distance between any unique choices of $r_i(u)$ is $\frac{\delta^{-p+1}}{n!}$. The following Lemma is applied to the construction of $\Psi$:

\begin{lemma}
For $k$ positive real numbers $a_1,a_2,...,a_k$, $\sum_{i=1}^{k} a_i$ is a unique descriptor of $(a_1,a_2,\ldots,a_k)$ if:
\begin{enumerate}
    \item For any $a_i$, there exists $r_i$ such that $a_i\leq r_i$,
    \item Let $s(i)$ be the minimal distance between any unique choices of $a_i$, then $s(i)>\sum_{j=1}^{i-1}r_j$ holds for any $i$.
\end{enumerate}
\end{lemma}

\begin{proof}
Assuming that there exists two groups of positive real numbers $\{a_i^{(1)}\}_{i=1}^k$ and $\{a_i^{(2)}\}_{i=1}^k$ which both satisfy conditions above and $\sum_{i=1}^k a_i^{(1)}=\sum_{i=1}^k a_i^{(2)}$. Besides, the two group of numbers are not totally equal correspondingly, which means there must exist one $l\in\{1,2,...,c\}$ such that $a_l^{(1)}\neq a_l^{(2)}$ and for any $j>l$, $a_j^{(1)}=a_j^{(2)}$ holds.

According to the second condition, $|a_l^{(1)}-a_l^{(2)}|>\sum_{j=1}^{l-1}r_j$ holds. Since $\sum_{i=1}^k a_i^{(1)}=\sum_{i=1}^k a_i^{(2)}$, then
\begin{align}
    |a_l^{(1)}-a_l^{(2)}|&=|\sum_{j=1}^{l-1}a_j^{(1)}-\sum_{j=1}^{l-1}a_j^{(2)}|\\
    &=|\sum_{j=1}^{l-1}(a_j^{(1)}-a_j^{(2)})|\\
    &\leq\sum_{j=1}^{l-1}|a_j^{(1)}-a_j^{(2)}|\\
    &\leq\sum_{j=1}^{l-1}r_j,
\end{align}
where the proof is completed by contradiction.
\end{proof}

Finally we consider the definition of $\Psi$. We have
\begin{align}
    \bar l_u=l_u+\sum_{i=1}^c n!\delta^{(3p+q)i}r_i(u),
\end{align}
and it can be concluded that
\begin{align}
    &1\leq l_u\leq \delta^{p+1},\\
    &\delta^{3ip+(i+1)q+1}\leq n!\delta^{(3p+q)i}r_i(u)\leq \delta^{(3i+1)p+(i+1)q},\text{ for }i=1,\ldots,c.
\end{align}
And the minimal distance between any unique choices of $n!\delta^{(3p+q)i}r_i(u)$ is 
\begin{align}
    s(i)=\delta^{(3i-1)p+iq+1}.
\end{align}
If $i=1$, then $s(1)=\delta^{2p+q+1}>\delta^{p+1}$; if $i=j+1$, then trivially we have
\begin{align}
    s(j+1)=\delta^{(3j+2)p+(j+1)q+1}>\sum_{s=1}^{j}\delta^{(3s+1)p+(s+1)q}.
\end{align}

Thus, according to the lemma above, $\bar l_u$ also becomes a unique descriptor of $\ldblbrace Y(u,i):i=1,\ldots,n\rdblbrace$, then it must be a unique descriptor of $\ldblbrace(\vx_u,\bm{\alpha}_{u,j}):j=1,\ldots,n\rdblbrace$. We also have $\bar l_u$ bounded as $\bar l_u<\delta^{(3c+1)p+(c+1)q-1}$, which completes the proof. \end{proof}

\paragraph{Step 4.} After the previous steps, $\bar l_u$ is the unique id for $\ldblbrace(\vx_j,\bm{\alpha}_{u,j}):j=1,\ldots,n\rdblbrace$, and we have $\bar l_u\in[0:\delta:\delta^{(3c+1)p+(c+1)q-1}-\delta].$ It can be observed that if we define $d'=(3c+1)p+(c+1)q-1$ and treat $d'$ as the "new" $d$, we can apply exactly the same methods in Appendix B.5 of \cite{yun2019transformers} to employ multiple selective shift operations and generate contextual embeddings for $\gH=[\ldblbrace(\vx_j,\bm{\alpha}_{i,j}):j=1,\ldots,n\rdblbrace]_{i=1,\ldots,n}$. Note that since we assume $\gX\subset (0,1)^{d\times n}$, only Category 1 and 2 (Appendix B.5 of \cite{yun2019transformers}) need to be considered.

\paragraph{Step 5.} Now with contextual embeddings $q(\gH)$, we can use methods in Appendix B.6 of \cite{yun2019transformers} to map every mapping values to the desired output computed by $\bar f$, which completes the proof.
\end{proof}

\subsection{Proof for Theorem \ref{thm3}}
\begin{proof}
Let the label mappings generated by $S'$-SEG-WL and $S$-SEG-WL at iteration $t$ be $g'_t$ and $g_t$ respectively. We denote the conditions in Equation \ref{thm3eq1} and \ref{thm3eq2} as $f_A=p_A(f_A')$ and $f_R=p_R(f_R')$. For graphs $G_v=(V_v,E_v)$ and $G_u=(V_u,E_u)$ ($G_v$ and $G_u$ may be the same graph), we first show that for any node $v\in V_v$ and $u\in V_u$ at iteration $t$, if $S'$-SEG-WL generates $g_t'(v)=g_t'(u)$, then $S$-SEG-WL also gets $g_t(v)=g_t(u)$. For $t=0$ this holds because if $g_0'(v)=g_0'(u)$ then $v$ and $u$ must have $h_0(v)=h_0(u)$ and $f_A'(v,G_v)=f_A'(u,G_u)$. Since $f_A=p_A(f_A')$, it means that $f_A(v,G_v)=f_A(u,G_u)$, which leads to $g_0(v)=g_0(u)$. Suppose this condition holds for iteration $0,1,\ldots,t$ and $g_{t+1}'(v)=g_{t+1}'(u)$. From the injectiveness of function $\Phi$, we have
\begin{align}
    \ldblbrace (g_t'(r),f_R'(v,r,G_v)):r\in V_v\rdblbrace=\ldblbrace (g_t'(r),f_R'(u,r,G_u)):r\in V_u\rdblbrace.
\end{align}

If two finite multisets are identical, then the elements in the two multisets can be matched in pairs. The condition $f_R=p_R(f_R')$ implies that for any $r,s$, $f_R'(v,r,G_v)=f_R'(u,s,G_u)\implies f_R(v,r,G_v)=f_R(u,s,G_u)$. Together with the assumption that $g_t'(v)=g_t'(u)$ implies $g_t(v)=g_t(u)$, we can conclude that
\begin{align}
    &(g_t'(r),f_R'(v,r,G_v))=(g_t'(s),f_R'(u,s,G_u))\implies\\
    &(g_t(r),f_R(v,r,G_v))=(g_t(s),f_R(u,s,G_u)).
\end{align}
Therefore, we have
\begin{align}
    \ldblbrace (g_t(r),f_R(v,r,G_v)):r\in V_v\rdblbrace=\ldblbrace (g_t(r),f_R(u,r,G_u)):r\in V_u\rdblbrace,
\end{align}
which directly leads to $g_{t+1}(v)=g_{t+1}(u)$. Then the proposition above is proved by induction. Now that for any iteration $t$ we have $g_t'(v)=g_t'(u)\implies g_t(v)=g_t(u)$, indicating that a mapping $\psi_t$ exists such that for any node $v$, $g_t(v)=\psi_t(g_t'(v))$.

Now consider two graphs $G_1$ and $G_2$ where $S$-SEG-WL decides them as non-isomorphic after $t$ iterations, then the multiset of all updated node labels $\ldblbrace g_t(v):v\in V\rdblbrace$ must be different for two graphs. Since $\ldblbrace g_t(v):v\in V\rdblbrace=\ldblbrace \psi_t(g_t'(v)):v\in V\rdblbrace$, $\ldblbrace g_t'(v):v\in V\rdblbrace$ must also be different for two graphs or we will reach a contradiction, which suggests that $S'$-SEG-WL distinguishes $G_1$ and $G_2$ after $t$ iterations.
\end{proof}

\subsection{Proof for Theorem \ref{thm4}}
\begin{proof}
The formal definition for WL test is presented in Appendix \ref{asec2}. Here we denote $N^+(v)=N(v)\cup\{v\}$ as the ego subgraph of node $v$. For label update of node $v$, the values of $\textit{Neighbor}_R$ divides the node set $V$ into three parts: the central node $v$, the neighborhood nodes $N(v)$ and nodes out of $v$'s ego subgraph $V\setminus N^+(v)$. Thus, the node label update function controlled by $\textit{Neighbor}_R$ can be viewed as
\begin{align}
    g_t(v)=\Phi(g_t(v),\ldblbrace g_t(r):r\in N(v)\rdblbrace,\ldblbrace g_t(s):s\in V\setminus N^+(v)\rdblbrace).
\end{align}

For the first part of the proof, we prove that \textit{Neighbor}-SEG-WL can distinguish any non-isomorphic graphs distinguishable by WL test. We first show that for any node $v,u$ at iteration $t$, if \textit{Neighbor}-SEG-WL generates $g_t(v)=g_t(u)$, then WL will obtain $w_t(v)=w_t(u)$. For $t=0$ this obviously holds. Suppose this condition holds for iteration $0,1,\ldots,t$ and $g_{t+1}(v)=g_{t+1}(u)$. From the injectiveness of function $\Phi$, we have
\begin{align}
    &(g_t(v),\ldblbrace g_t(r):r\in N(v)\rdblbrace,\ldblbrace g_t(s):s\in V_v\setminus N^+(v)\rdblbrace)\\
    =&(g_t(u),\ldblbrace g_t(r):r\in N(u)\rdblbrace,\ldblbrace g_t(s):s\in V_u\setminus N^+(u)\rdblbrace),
\end{align}
where $V_v$ is the node set of graph that $v$ belongs to, which is the same for $V_u$.
Slicing the two equivalent tuples above will also get equivalent results, as
\begin{align}
    (g_t(v),\ldblbrace g_t(r):r\in N(v)\rdblbrace)=(g_t(u),\ldblbrace g_t(r):r\in N(u)\rdblbrace).
\end{align}

Therefore we have $w_{t+1}(v)=w_{t+1}(u)$, and the proposition above is proved by induction. Now that for any iteration $t$ we have $g_t(v)=g_t(u)\implies w_t(v)=w_t(u)$, indicating that a mapping $\psi_t$ exists such that for any node $v$, $w_t(v)=\psi_t(g_t(v))$.

Consider two graphs $G_1$ and $G_2$ where WL decides them as non-isomorphic after $t$ iterations, then the multiset of all updated node labels $\ldblbrace w_t(v):v\in V\rdblbrace$ must be different for two graphs. Since $\ldblbrace w_t(v):v\in V\rdblbrace=\ldblbrace \psi_t(g_t(v)):v\in V\rdblbrace$, $\ldblbrace g_t(v):v\in V\rdblbrace$ must also be different for two graphs, which suggests that \textit{Neighbor}-SEG-WL can distinguish $G_1$ and $G_2$ after $t$ iterations.

In the second part of the proof we only need to show that any non-isomorphic graphs indistinguishable by WL test can not be distinguished by \textit{Neighbor}-SEG-WL. Suppose there are two graphs $G_1=(V_1,E_1)$ and $G_2=(V_2,E_2)$ that WL test cannot distinguish and the iteration converges at iteration $t$. Then for any $v,u\in V_1$, $w_t(v)=w_t(u)$ implies $w_{t+1}(v)=w_{t+1}(u)$ (the same for $V_2$), and there exists a bijective mapping $\theta:V_1\to V_2$ such that for any $v\in V_1$, $w_t(v)=w_t(\theta(v))$ and $w_{t+1}(v)=w_{t+1}(\theta(v))$. Since $w_t$ can be viewed as an absolute structural encoding function, we denote $\textit{Neighbor}^+=(w_t,\textit{Neighbor}_R)$ and $\textit{Neighbor}^+$ must be more powerful than \textit{Neighbor} according to Theorem \ref{thm3} because $\textit{Neighbor}^+\succeq\textit{Neighbor}$. Let $g^+$ be the label mapping generated by $\textit{Neighbor}^+$-SEG-WL on $G_1$ and $G_2$, and we may assume without generality that $g^+_0=w_t$. For node $v\in V_1$, its first updated label is computed by
\begin{align}
    g_1^+(v)=\Phi(w_t(v),\ldblbrace w_t(r):r\in N(v)\rdblbrace,\ldblbrace w_t(s):s\in V_1\setminus N^+(v)\rdblbrace).
\end{align}
Consider $v,u\in V_1$ where $w_t(v)=w_t(u)$. According to the definition of WL test, $w_{t+1}(v)=w_{t+1}(u)$ implies $\ldblbrace w_t(r):r\in N(v)\rdblbrace=\ldblbrace w_t(r):r\in N(u)\rdblbrace$. And because $v$ and $u$ belongs to the same graph, we also have $\ldblbrace w_t(s):s\in V_1\setminus N^+(v)\rdblbrace=\ldblbrace w_t(s):s\in V_1\setminus N^+(u)\rdblbrace$. This results in $g_1^+(v)=g_1^+(u)$.

Next we consider $v\in V_1$ and $\theta(v)\in V_2$ where $w_t(v)=w_t(\theta(v))$ and $w_{t+1}(v)=w_{t+1}(\theta(v))$. According to the definition of WL test, $w_{t+1}(v)=w_{t+1}(\theta(v))$ implies $\ldblbrace w_t(r):r\in N(v)\rdblbrace=\ldblbrace w_t(r):r\in N(\theta(v))\rdblbrace$. Since WL test can not distinguish $G_1$ and $G_2$, we have $\ldblbrace w_t(s):s\in V_1\rdblbrace=\ldblbrace w_t(s):s\in V_2\rdblbrace$, indicating that $\ldblbrace w_t(s):s\in V_1\setminus N^+(v)\rdblbrace=\ldblbrace w_t(s):s\in V_2\setminus N^+(\theta(v))\rdblbrace$, which shows $g_1^+(v)=g_1^+(\theta(v))$.

Together with statements above, for any $v,u\in V_1$ with $g_0^+(v)=g_0^+(u)$, we have $g_1^+(v)=g_1^+(u)$ and $g_1^+(v)=g_1^+(\theta(v))$. As $\theta$ is a bijective mapping, we can conclude that a mapping $\mu$ exists such that for any $v\in V_1\cup V_2$, $g_1^+(v)=\mu(g_0^+(v))$, which tells us that $\ldblbrace g_1^+(s):s\in V_1\rdblbrace=\ldblbrace g_1^+(s):s\in V_2\rdblbrace$ and $\textit{Neighbor}^+$-SEG-WL has not update any useful information in its first iteration. Therefore, we can see that $\ldblbrace g_t^+(s):s\in V_1\rdblbrace=\ldblbrace g_t^+(s):s\in V_2\rdblbrace$ for any $t$ by induction, then $\textit{Neighbor}^+$-SEG-WL can not distinguish $G_1$ and $G_2$. Because $\textit{Neighbor}^+$-SEG-WL is more powerful than $\textit{Neighbor}$-SEG-WL, $\textit{Neighbor}$-SEG-WL also can not distinguish the two graphs, meaning that any non-isomorphic graphs indistinguishable by WL test can not be distinguished by $\textit{Neighbor}$-SEG-WL, which completes the proof.
\end{proof}

\subsection{Proof for Theorem \ref{thm5}}
\begin{proof}
We can easily show that \textit{SPD}-SEG-WL is more powerful than \textit{Neighbor}-SEG-WL using Theorem \ref{thm3} since two nodes are linked if there shortest path distance is 1. And according to Theorem \ref{thm4}, \textit{Neighbor}-SEG-WL is as powerful as WL, then \textit{SPD}-SEG-WL is more powerful than WL.

Figure \ref{afig1} below shows a pair of graphs that can be distinguished by \textit{SPD}-SEG-WL but not WL, which completes the proof.
\end{proof}

\begin{figure}[h]
\centering
\includegraphics[width=0.5\linewidth]{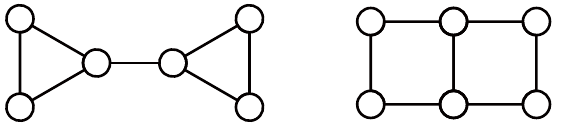}
\caption{Two graphs that can be distinguished by \textit{SPD}-SEG-WL but not WL.}
\label{afig1}
\end{figure}

\subsection{Proof for Proposition \ref{mpro1}}
\begin{proof}
Let $C_l$ denote the cycle graph of length $l$. Then consider two graphs $G_1$ and $G_2$, where $G_1$ consists of $2k+4$ identical $C_{2k+3}$ graphs, and $G_2$ consists of $2k+3$ identical $C_{2k+4}$ graphs. $G_1$ and $G_2$ have the same number of nodes, and the induced $k$-hop neighborhood of any node in either of the two graphs is simply a path of length $2k+1$. As a result, for structural encoding scheme $S$ with $k$-hop receptive field, $S$-SEG-WL generates identical labels for every node in the two graphs, making $G_1$ and $G_2$ indistinguishable for $S$-SEG-WL. However, in $G_2$ there exists shortest paths of length $k+2$ while $G_1$ not, so \textit{SPD}-SEG-WL can distinguish the two graphs.
\end{proof}

\subsection{Proof for Theorem \ref{thm6}}
\begin{proof}
Considering $\textit{SPD}_R$ is the first dimension of $\textit{SPIS}_R$, we have $\textit{SPIS}\succeq\textit{SPD}$ and we can prove $\textit{SPIS}$-SEG-WL is more powerful than $\textit{SPD}$-SEG-WL according to Theorem \ref{thm3}.

Figure \ref{afig2} below shows a pair of graphs that can be distinguished by \textit{SPIS}-SEG-WL but not \textit{SPD}-SEG-WL. It is trivial to verify that \textit{SPD}-SEG-WL can not distinguish them. For \textit{SPIS}-SEG-WL, to understand this, Figure \ref{afig2} colors examples of SPIS between non-adjacent nodes in the two graphs, where the nodes at two endpoints are colored as red. In the first graph, every SPIS between non-adjacent nodes has 3 nodes, but in the second graph there exists SPIS between non-adjacent nodes that has 4 nodes, so \textit{SPIS}-SEG-WL can distinguish them.
\end{proof}

\begin{figure}[h]
\centering
\includegraphics[width=0.5\linewidth]{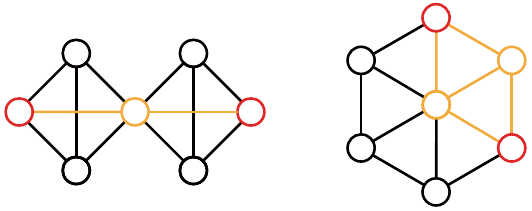}
\caption{Two graphs that can be distinguished by \textit{SPIS}-SEG-WL but not \textit{SPD}-SEG-WL.}
\label{afig2}
\end{figure}

\subsection{Proof for Proposition \ref{mpro2}}
\begin{proof}
It is trivial to verify that regular graphs with different parameters can be distinguished by WL, so we focus on strongly regular graphs with the same $n$ and $k$ but different $\lambda$ and $\mu$. For $\text{SRG}(n,k,\lambda,\mu)$, since every non-adjacent pair of nodes has $\mu$ neighbors in common, the SPIS between evry non-adjacent pair of nodes will have $\mu+2$ nodes, which implies that \textit{SPIS}-SEG-WL can distinguish strongly regular graphs with different $n,k,\mu$. Besides, the four parameters of strongly regular graphs are not independent, they satisfy
\begin{align}
    \lambda=k-1-\frac \mu k(n-k-1),
\end{align}
so \textit{SPIS}-SEG-WL can distinguish strongly regular graphs with different parameters.
\end{proof}

\subsection{Proof for Proposition \ref{mpro3}}
\begin{proof}
Figure \ref{afig3} below shows a pair of graphs that can be distinguished by \textit{SPIS}-SEG-WL but not 3-WL. The two graphs, named as the Shrikhande graph and the Rook's $4\times 4$ graph, are both $\text{SRG}(16,6,2,2)$ and the most popular example for indistinguishability with 3-WL \citep{arvind2020weisfeiler}. To show they can be distinguished by \textit{SPIS}-SEG-WL, Figure \ref{afig3} also colors examples of SPIS between non-adjacent nodes, where the nodes at two endpoints are colored as red. In the second graph (the Shrikhande graph), one can verify that every SPIS between non-adjacent nodes has 4 nodes and 4 edges, but in the first graph (the Rook's $4\times 4$ graph) there exists SPIS between non-adjacent nodes that has 5 edges, making \textit{SPIS}-SEG-WL capable of distinguishing them.
\end{proof}

\begin{figure}[h]
\centering
\includegraphics[width=0.8\linewidth]{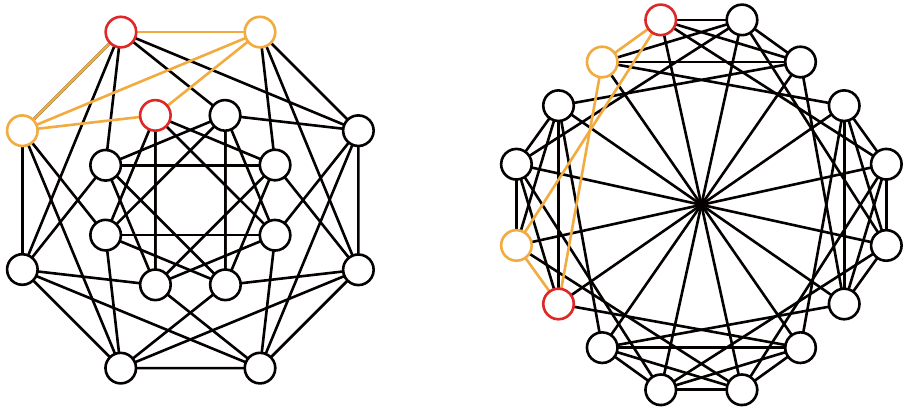}
\caption{Two graphs (the Shrikhande graph and the Rook's $4\times 4$ graph) that can be distinguished by \textit{SPIS}-SEG-WL but not 3-WL.}
\label{afig3}
\end{figure}

\section{More Discussions}
\label{asec3}
\paragraph{Isomorphic Structural Encodings.} For any structural encoding function, we say $f_A$ is \textit{isomorphic} to $f_A'$ if there exists a bijective mapping $p$ such that $f_A=p(f_A')$, which is the same for $f_R$. It is trivial to conclude that isomorphic structural encodings have the same expressive power.

\paragraph{Reduction of Absolute Structural Encodings.} It can be observed that for structural encoding scheme $S=(f_A,f_R)$, $f_A$ and $f_R$ may express overlapping information and can be reduced to form a more concise representation. Since the principal phase of SEG-WL test is the label update controlled by relative structural encodings, we focus on the case where $f_A$ can be deduced from $f_R$, and we can reduce $f_A$ to eliminate redundant information, which is defined as

\begin{definition}[Reduction of Absolute Structural Encodings]
A structural encoding scheme $S=(f_A,f_R)$ can be reduced to $S'=(\textit{id}_A,\textit{f}_R)$ if there exists mapping $p$ such that for any $G=(V,E)$ and $v\in V$ we have
\begin{align}
    f_A(v,G)=p(\ldblbrace f_R(v,u,G):u\in V\rdblbrace)
\end{align}
\end{definition}

\begin{proposition}
If structural encoding scheme $S$ can be reduced to $S'$, then two graphs can be distinguished by $S$-SEG-WL if and only if they are distinguishable by $S'$-SEG-WL.
\label{apro8}
\end{proposition}

\begin{proof}
According to Theorem \ref{thm3}, $S$-SEG-WL is more powerful than $S'$-SEG-WL, thus we only need to prove that any graphs distinguishable by $S$-SEG-WL can be distinguished by $S'$-SEG-WL. Let the label mappings generated by $S'$-SEG-WL and $S$-SEG-WL at iteration $t$ be $g'_t$ and $g_t$ respectively. For graphs $G_v=(V_v,E_v)$ and $G_u=(V_u,E_u)$ ($G_v$ and $G_u$ may be the same graph), we first show that for any node $v\in V_v$ and $u\in V_u$ at iteration $t$, if $S'$-SEG-WL generates $g_{t+1}(v)=g_{t+1}(u)$, then $S$-SEG-WL also gets $g_t(v)=g_t(u)$. For $t=0$, from the injectiveness of $\Phi$ we have 
\begin{align}
    \ldblbrace (g_0'(r),f_R'(v,r,G_v)):r\in V_v\rdblbrace=\ldblbrace (g_0'(r),f_R'(u,r,G_u)):r\in V_u\rdblbrace.
\end{align}
Accordingly, we have
\begin{align}
    \ldblbrace f_R'(v,r,G_v):r\in V_v\rdblbrace=\ldblbrace f_R'(u,r,G_u):r\in V_u\rdblbrace,
\end{align}
which directly leads to $f_A(v,G_v)=f_A(u,G_u)$. According to the definition of $S'$, we have $g_0(v)=g_0(u)$. Suppose this condition holds for iteration $0,\ldots,t$ and $g_{t+1}'(v)=g_{t+1}'(u)$. From the injectiveness of function $\Phi$, we have
\begin{align}
    \ldblbrace (g_t'(r),f_R(v,r,G_v)):r\in V_v\rdblbrace=\ldblbrace (g_t'(r),f_R(u,r,G_u)):r\in V_u\rdblbrace.
\end{align}
According to the assumption that $g_t'(v)=g_t'(u)$ implies $g_{t-1}(v)=g_{t-1}(u)$, we can infer that
\begin{align}
    \ldblbrace (g_{t-1}(r),f_R(v,r,G_v)):r\in V_v\rdblbrace=\ldblbrace (g_{t-1}(r),f_R(u,r,G_u)):r\in V_u\rdblbrace,
\end{align}
which directly leads to $g_t(v)=g_t(u)$. Then the proposition above is proved by induction. Now that for any iteration $t$ we have $g_{t+1}'(v)=g_{t+1}'(u)\implies g_t(v)=g_t(u)$, indicating that a mapping $\psi_t$ exists such that for any node $v$, $g_t(v)=\psi_t(g_{t+1}'(v))$.

Now consider two graphs $G_1$ and $G_2$ where $S$-SEG-WL decides them as non-isomorphic after $t$ iterations, then the multiset of all updated node labels $\ldblbrace g_t(v):v\in V\rdblbrace$ must be different for two graphs. Since $\ldblbrace g_t(v):v\in V\rdblbrace=\ldblbrace \psi_t(g_{t+1}'(v)):v\in V\rdblbrace$, $\ldblbrace g_{t+1}'(v):v\in V\rdblbrace$ must also be different for two graphs or we will reach a contradiction, which suggests that $S'$-SEG-WL distinguishes $G_1$ and $G_2$ after $t+1$ iterations.
\end{proof}

The proposition above guarantees that the reduction of redundant encodings will not influence the expressive power of corresponding SEG-WL test. For example, since the degree of nodes can be obtained by counting its neighbors, then $(\textit{Deg}_A,\textit{Neighbor}_R)$ can be reduced to $(\textit{id}_A,\textit{Neighbor}_R)$. 

\section{Connections between SEG-WL Test and Previous Graph Transformers}
\label{asec4}
As we have discussed above, SEG-WL test is capable of characterizing the expressive power of most graph Transformers, and here we will present some examples. Note that in the scope of this paper, we only consider simple undirected graphs with node features.

\paragraph{Graphormer \citep{ying2021transformers}.} The Graphormer model utilizes three types of structural encodings: \textit{Centrality Encoding} that encodes node degrees, \textit{Spatial Encoding} that encodes the structural relation between nodes via shortest path distance, and \textit{Edge Encoding} that captures information of edges that connect two nodes (which we do not consider since it relates to edge feature). The \textit{Centrality Encoding} corresponds to the $\textit{Deg}_A$ absolute structural encoding we discuss in Section \ref{sec41}, and the \textit{Spatial Encoding} is equivalent to the shortest path distance encoding $\textit{SPD}_R$ in Section \ref{sec41}. Therefore, similar to the proof for Proposition \ref{apro0}, we can prove that the expressivity of Graphormer with two types of structural encoding above can be characterized with \textit{Graphormer}-SEG-WL, where
\begin{align}
    \textit{Graphormer}=(\textit{Deg}_A,\textit{SPD}_R).
\end{align}
According to Proposition \ref{apro8}, the \textit{Graphormer} encoding above can be reduced to $\textit{SPD}=(\textit{id}_A,\textit{SPD}_R)$ since the degree of node $v$ can be inferred from the number of node $v$ such that $\textit{SPD}(v,u)=1$. Thus, the expressivity of Graphormer can be characterized with $\textit{SPD}$-SEG-WL. According to our analysis in Section \ref{sec51}, $\textit{SPD}$-SEG-WL is strictly more powerful than WL and has unique expressive power elaborated by Proposition \ref{mpro1}.

\paragraph{SEG-WL \citep{dwivedi2020generalization} and SAN \citep{kreuzer2021rethinking}.} SEG-WL and SAN both employ Laplacian eigenvalues and eigenvectors as absolute structural encodings, and during Transformer layers the embedding update strategy is determined by link connections. For both models, it can be easily verified that $\textit{Laplacian}_A^k$ below characterizes their absolute structural encodings:
\begin{align}
    \textit{Laplacian}_A^k(v,G)=(\Lambda_G^k, \lambda_v^k),
\end{align}
where $\Lambda_G^k$ is the $k$ smallest Laplacian eigenvalues of graph $G$, $\lambda_v^k$ is the Laplacian eigenvector of $v$ in $G$ corresponding to $\Lambda_G^k$, and every $\textit{Laplacian}_A^k(v,G)$ comes from a \textit{deterministic factorization policy for graph Laplacian matrix}. As for relative structural encoding, since during Transformer layers both models only consider if two nodes are linked, we can conclude that $\textit{Neighbor}_R$ summarizes the expressivity of embedding update process. Therefore, $\textit{Laplacian}^k$-SEG-WL is an upper bound on the expressivity of SAN and SEG-WL model, where
\begin{align}
    \textit{Laplacian}^k=(\textit{Laplacian}_A^k,\textit{Neighbor}_R).
\end{align}
It is quite difficult to accurately analyze the expressive power of $\textit{Laplacian}^k$ since it relates to the sign invariance of Laplacian eigenvectors and contents of spectral graph theory. However, since $\textit{Laplacian}^k$ only involves the $\textit{Neighbor}_R$ relative encoding, our Theorem \ref{thm4} shows that for SAN and SEG-WL, the exploitation of Transformer network results in no improvement on the structural expressive power when comparing with GNNs using $\textit{Laplacian}_A^k$ as additional node features.

\paragraph{Gophormer \citep{zhao2021gophormer}.} Gophormer is a scalable graph Transformer model for node classification with proximity-enhanced multi-head attention (PE-MHA) as the core module for learning graph structure. When analyzing the structural expressive power of Gophormer, the global nodes added to represent global information are ignored. It can be concluded that the following $\textit{Proximity}^k_R$ relative structural encoding characterizes the expressivity of PE-MHA in Gophormer:
\begin{align}
    \textit{Proximity}^k_R(v_i,u_j,G)=(\mI(i,j),\tilde\mA(i,j),\ldots,\tilde\mA^k(i,j)),
\end{align}
where $\mI$ is the identity matrix, and $\tilde \mA=\text{Norm}(\mA+\mI)$ is the normalized adjacency matrix with self-loop. Since Gophormer employs no absolute structural encoding, $\textit{Proximity}^k$-SEG-WL describes the expressivity of Gophormer, where $\textit{Proximity}^k=(\textit{id}_A,\textit{Proximity}^k_R)$. 

As for any $v_i,v_j$, $\mA(i,j)$ can be inferred from $(\mI(i,j),\tilde\mA(i,j))$, the $\textit{Proximity}^k$ structural encoding is more expressive than \textit{Neighbor} when $k\leq 1$. As a result, according to Theorem \ref{thm4}, $\textit{Proximity}^k$-SEG-WL is more powerful than WL, and one can easily verify that two graphs in Figure \ref{afig1} can be distinguished by $\textit{Proximity}^k$-SEG-WL. Therefore, we can conclude that Gophormer with $\textit{Proximity}^k$ encoding is strictly more powerful than WL.

\paragraph{SAT \citep{chen2022structure}} SAT propose the Structure-Aware Transformer with its new self-attention mechanism which incorporates structural information into the original self-attention by extracting a subgraph representation rooted at each node using GNNs before computing the attention. Theoretical results in the SAT paper guarantees that SAT is at least as expressive as the GNN subgraph extractor, and using SEG-WL test we will arrive at the similar result. In the framework of SEG-WL test, regardless of absolute structural encoding, SAT model incorporates the node features generated by GNNs as relative structural encoding at each structure-aware attention:

\begin{align}
    &\textit{SAT}_R^\text{ subtree}(v,u,G)=(\text{GNN}_G^{(k)}(v),\text{GNN}_G^{(k)}(u))\\
    &\text{ ($k$-subtree GNN extractor)},\\
    &\textit{SAT}_R^\text{ subgraph}(v,u,G)=(\sum_{u\in N_k(v)}\text{GNN}_G^{(k)}(u),\sum_{r\in N_k(u)}\text{GNN}_G^{(k)}(r)),\\
    &\text{ ($k$-subgraph GNN extractor)}.
\end{align}
For $k$-subtree GNN extractor, considering that $\text{GNN}_G^{(k)}(v)$ can be inferred from $\ldblbrace\textit{SAT}_R^\text{ subtree}(v,u,G):u\in V\rdblbrace$ by choosing the first element of each tuple, with proposition \ref{apro8} we can conclude that $\text{SAT}^\text{ subtree}$ can be viewed as having absolute structural encoding generated by $\text{GNN}_G^k$, which is the same for $k$-subgraph GNN extractor. Therefore, $\textit{SAT}^\text{ subtree}$-SEG-WL is more powerful than $\phi(v,G)=\text{GNN}^{(k)}_G(v)$, and $\textit{SAT}^\text{ subgraph}$-SEG-WL is more powerful than $\phi(v,G)=\sum_{u\in N_k(v)}\text{GNN}_G^{(k)}(u)$, which shows that the expressivity upper bound of SAT is more powerful than its GNN feature extractor.
\section{Graph Representation Learning Experiments}
\label{asec5}

\subsection{Datasets}
\label{asec51}

\begin{table*}[htbp]
    \begin{tabular}{l|cccccc}
        \toprule
        Datasets & \#Graphs & \#Nodes & \#Node Attributes & \#Edges & \#Edge Attributes & \#Tasks \\
        \midrule
        ZINC(subset) & 12,000 & 277920 & 1 & 597960 & 1 & 1  \\
        QM9          & 130831 & 2359210& 11& 4883516& 4 & 12 \\
        QM8          & 21786  & 169339 & 79& 352356 & 10& 16 \\
        ESOL         & 1128   & 14991  & 9 & 15428  & 3 & 1  \\
        \bottomrule
    \end{tabular}
    \caption{Statics for graph regression datasets.}
    \label{regression}
    
\end{table*}

\begin{table*}[htbp]
    \begin{tabular}{l|cccccc}
        \toprule
        Datasets & \#Graphs & \#Nodes & \#Node Attributes & \#Edges & \#Edge Attributes & \#Classes \\
        \midrule
        PTC-MR   & 344  & 4015    & 18  & 10108 & 4   & 2   \\
        MUTAG    & 188  & 3371    & 7   & 7442  & 4   & 2   \\
        COX2     & 467  & 19252   & 35  & 40578 & -   & 2   \\
        PROTEINS & 1113 & 43471   & 3   & 162088& -   & 2   \\
        \bottomrule
    \end{tabular}
    \caption{Statics for graph classification datasets.}
    \label{classification}
    
\end{table*}

Statistics of the datasets used in this work are summarized in Table \ref{regression} and \ref{classification}. 

\subsection{Settings}

\label{asec52}
\subsubsection{Graphormer and GraphGPS Variants}
\label{asec521}
\paragraph{Model Description.} In graph representation learning experiments, We use four Graphormer variants based on four structural encoding schemes discussed in the main paper: \textit{SPIS}, \textit{SPD}, \textit{Neighbor} and \textit{id}. For Graphormer-\textit{SPIS}, to incorporate the extra structural information encoded by \textit{SPIS} encoding while not making significant changes to the model architecture, we replace the spatial encoding $b_{\textit{SPD}_R(v,u,G)}$ in Graphormer with $b_{\textit{SPD}_R(v,u,G)}+\text{Linear}(|V_{\textit{SPIS}(v,u)}|,|E_{\textit{SPIS}(v,u)}|)$, and keep the remaining network components unchanged. Graphormer-\textit{SPD} is basically the original Graphormer architecture. In Graphormer-\textit{Neighbor}, we remove the edge encoding since it contains information beyond the neighborhood connections, and replace the spatial encoding $b_{\textit{SPD}_R(v,u,G)}$ in Graphormer with $b_{\textit{Neighbor}_R(v,u,G)}$. Similarly, for Graphormer-\textit{id}, we remove the centrality encoding and edge encoding, and substitute the spatial encoding $b_{\textit{SPD}_R(v,u,G)}$ in Graphormer with $b_{\textit{id}_R(v,u,G)}$. For GraphGPS, we use the optimal settings reported by the original paper on ogb-PCQM4M dataset.

\begin{table*}[htbp]
    \begin{tabular}{l|cccccccc}
        \toprule
        & ZINC     & QM9      & QM8      & ESOL     & PTC-MR   & MUTAG    & COX2     & PROTEINS \\
        \midrule
        peak\_learning\_rate     & 2e-4 & 3e-4 & 3e-4 & 5e-4 & 0.01     & 0.01     & 0.01     & 0.01     \\
        end\_learning\_rate      & 1e-9 & 1e-9 & 1e-9 & 1e-9 & 1e-9 & 1e-9 & 1e-9 & 1e-9 \\
        hidden\_dim              & 80       & 512      & 256      & 256      & 256      & 256      & 256      & 256      \\
        ffn\_dim                 & 80       & 512      & 256      & 256      & 256      & 256      & 256      & 256      \\
        weight\_decay            & 0.01     & 0.0      & 0.0    & 0.0        & 0.0      & 0.0      & 0.0      & 0.0        \\
        input\_dropout\_rate     & 0.1      & 0.0        & 0.0        & 0.0     & 0.0  & 0.0        & 0.0        & 0.0        \\
        attention\_dropout\_rate & 0.1      & 0.0      & 0.1      & 0.1      & 0.1      & 0.1      & 0.1      & 0.1      \\
        dropout\_rate            & 0.1      & 0.0        & 0.1      & 0.1      & 0.1      & 0.1      & 0.1      & 0.1      \\
        num\_layers              & 12       & 20       & 6        & 16       & 16       & 16       & 16       & 16       \\
        num\_heads               & 8        & 32       & 16       & 16       & 16       & 16       & 16       & 16       \\
        \bottomrule
    \end{tabular}
    \caption{Model configurations and hyper-parameters of Graphormer with different types of structural encoding.}
    \label{hyper_graphormer}

\end{table*}

\paragraph{Model Configurations.}
We report the detailed hyper-parameter settings used for training the Graphormer variants in Table \ref{hyper_graphormer}. We use the source code provided by \citep{ying2021transformers} (MIT 2.0 license) and use AdamW \citep{loshchilov2018decoupled} as optimizer and linear decay as learning rate scheduler. All models are trained on 2 NVIDIA RTX 3090 GPUs for up to 12 hours.

\subsubsection{Baselines}
\label{asec522}
\begin{table*}[htbp]
    \begin{tabular}{l|cccccccc}
        \toprule
        & ZINC     & QM9      & QM8      & ESOL     & PTC-MR   & MUTAG    & COX2     & PROTEINS \\
        \midrule
        peak\_learning\_rate       & 3e-4 & 3e-4 & 3e-4 & 1e-3 & 3e-4 & 3e-4 & 3e-4 & 3e-4 \\
        end\_learning\_rate        & 1e-9 & 1e-9 & 1e-5 & 1e-9 & 1e-9 & 1e-9 & 1e-9 & 1e-9 \\
        hidden\_dim                & 256      & 256      & 256      & 512      & 256      & 256      & 256      & 256      \\
        weight\_decay              & 0.01     & 0.01     & 0.0        & 0.01     & 0.01     & 0.01     & 0.01     & 0.01     \\
        input\_dropout\_rate       & 0.0        & 0.0        & 0.1      & 0.1      & 0.0    & 0.0   & 0.0        & 0.0        \\
        dropout\_rate              & 0.1      & 0.1      & 0.1      & 0.1      & 0.1      & 0.1      & 0.1      & 0.1      \\
        num\_layers                & 16       & 16       & 16       & 5        & 5        & 5        & 5        & 5        \\
        num\_heads(only for GAT)   & 4        & 4        & 4        & 4        & 4        & 4        & 4        & 4        \\
        \bottomrule
    \end{tabular}
    \caption{Model configurations and hyper-parameters of GNN baselines.}
    \label{hyper_gnn}

\end{table*}

\paragraph{Model Configurations.}
We report the detailed hyper-parameter settings used for training GNN baselines including GCN \citep{kipf2016semi}, GAT \citep{velivckovic2017graph}, GIN \citep{xu2018powerful} and GraphSAGE \citep{hamilton2018inductive} in Table \ref{hyper_gnn}. During training stage, we use AdamW \citep{loshchilov2018decoupled} as optimizer and decay the learning rate with a cosine annealing utilized in \citep{loshchilov2016sgdr}. All models are trained on 2 NVIDIA RTX 3090 GPUs until convergence for up to 12 hours.

For SAT \citep{chen2022structure} model, it has substantially higher complexity than all proposed methods and baselines with its GNN-based feature extractor. We follow the instructions and run the code in \url{https://github.com/BorgwardtLab/SAT} on ZINC dataset. Due to limitations on computational resources, to give a fair comparison, we run the model for 3.5 days with almost 1000 epochs, and report the best performance.

\subsection{Performances on QM9}
\label{asec53}

\begin{table*}[htbp]
    \begin{tabular}{l|c|cccc}
        \toprule
         \multirow{2}*{Task} & \multirow{2}*{Unit} & \multicolumn{4}{c}{MAE} \\ 
         & & Graphormer-\textit{id} & Graphormer-\textit{Neighbor} & Graphormer-\textit{SPD} & Graphormer-\textit{SPIS} \\
        \midrule
        $\mu$  & \text{D} & 8.1654\small$\pm$0.1095 & 0.6926\small$\pm$1.646e-4 & 0.3688\small$\pm$3.010e-4 & 0.3536\small$\pm$3.727e-4 \\
        $\alpha$ & $a_0^3$ & 24.562\small$\pm$0.1815 & 0.8597\small$\pm$6.886e-4 & 0.2417\small$\pm$8.542e-7 & 0.2365\small$\pm$1.105e-3 \\
        $\epsilon_\text{HOMO}$ & \text{eV} & 1.5222\small$\pm$0.0283 & 0.1962\small$\pm$3.667e-4 & 0.0683\small$\pm$2.186e-5 & 0.0664\small$\pm$5.848e-5 \\
        $\epsilon_\text{LUMO}$ & \text{eV} & 4.3868\small$\pm$0.2717 & 0.2644\small$\pm$5.850e-5 & 0.0699\small$\pm$1.036e-5 & 0.0686\small$\pm$9.445e-5  \\
        $\Delta\epsilon$ & \text{eV} & 0.6235\small$\pm$0.0126 & 0.3407\small$\pm$4.459e-4 & 0.0933\small$\pm$1.420e-4 & 0.0904\small$\pm$2.811e-4 \\
        $\langle R^2\rangle$  & $a_0^2$ & 166.64\small$\pm$12.339 & 76.885\small$\pm$2.309e-2 & 18.774\small$\pm$7.047e-2 & 18.174\small$\pm$3.046e-2 \\
        $\text{ZPVE}$ & \text{eV} & 1.3654\small$\pm$0.0391 & 0.0165\small$\pm$3.954e-6 & 0.0061\small$\pm$2.012e-4 & 0.0055\small$\pm$5.311e-7 \\
        $U_0$  & \text{eV} & 3457.2\small$\pm$274.96 & 1.0558\small$\pm$8.925e-4 & 3.8210\small$\pm$7.458e-2 & 2.1069\small$\pm$3.581e-4 \\
        $U$ & \text{eV} & 2041.3\small$\pm$47.641 & 1.0552\small$\pm$2.932e-4 & 3.8882\small$\pm$2.049e-1 & 2.1069\small$\pm$3.694e-4 \\
        $H$  & \text{eV} & 3593.4\small$\pm$31.424 & 1.0540\small$\pm$6.737e-4 & 3.7888\small$\pm$1.232e-1 & 2.1007\small$\pm$4.798e-4 \\
        $G$ & \text{eV}  & 1468.9\small$\pm$97.816 & 1.0505\small$\pm$6.409e-4 & 3.8175\small$\pm$1.508e-1 & 2.0994\small$\pm$3.115e-4 \\
        $c_\text{v}$   & $\frac{\text{cal}}{\text{mol K}}$ & 5.4585\small$\pm$0.1456 & 0.4510\small$\pm$1.725e-4 & 0.1034\small$\pm$5.555e-5 & 0.1027\small$\pm$6.856e-7 \\
        \bottomrule
    \end{tabular}
    \caption{Performance on QM9, reported by separate tasks.}
    \label{QM9_mae}
    
\end{table*}

Here we additionally report the performance of Graphormer variants over 12 tasks individually on QM9 dataset in Table \ref{QM9_mae}.

\subsection{Code}
The experiment code is available in \url{https://drive.google.com/file/d/1umXMdH1wz3wk3dxZ6XOoe0eys7x9AoW8/view?usp=share_link}.

\section{Limitations and Possible Negative Societal Impacts}
\label{asec6}

\paragraph{Limitations.} It is well-known that self-attention in Transformer network has quadratic complexity with respect to the input size, and since SEG-WL test is proposed to characterize the expressivity of graph Transformers, it inherits this complexity issue and each label update iteration of SEG-WL test costs $O(n^2)$ complexity (equivalent to 2-WL), where $n$ is the input graph size. Besides, the structural encodings may be computed by algorithms with relative high complexity, like \textit{SPD} which is obtained by the $O(n^3)$ Floyd-Warshall algorithm, and in Appendix \ref{asec3} we formulate the complexity of proposed \textit{SPIS} as $O(n^3+n^2t^2)$. Still, we believe it is worth studying the expressive power of graph Transformers despite these limitations on complexity. It is shown that the global receptive field brought by self-attention can lead to higher performance than traditional GNNs on real-world benchmarks \citep{ying2021transformers}. Additionally, as Transformer gain popularity in multiple areas of machine learning, the complexity issue of Transformers can be mostly resolved by low-complexity self-attention techniques and modern computational devices specially optimized for Transformers. Therefore, together with our theoretical results which show graph Transformers can exhibit outstanding expressive power, we believe Transformers will be widely used in graph machine learning due to their performance and expressivity, despite their higher complexity than GNNs.

\paragraph{Ethic Statement and Possible Negative Societal Impacts.} This work is a foundational research on the expressivity of graph Transformers and is not tied to any particular applications. Therefore, our work may have potential negative societal impacts with malicious use of graph neural models (like generating fake profiles) or environmental impact (like training huge graph Transformers).

\end{document}